\documentclass{article}

\usepackage[nonatbib, main, final]{neurips_2026}

\usepackage[numbers]{natbib}
\usepackage[utf8]{inputenc} %
\usepackage[T1]{fontenc}    %
\usepackage{hyperref}       %
\usepackage{url}            %
\usepackage{booktabs}       %
\usepackage{amsfonts}       %
\usepackage{nicefrac}       %
\usepackage{microtype}      %
\usepackage{xcolor}         %
\usepackage{float}
\usepackage{multirow}
\usepackage{wrapfig}

\usepackage{amsmath}
\DeclareMathOperator*{\argmax}{arg\,max}
\DeclareMathOperator*{\argmin}{arg\,min}
\newcommand{\norm}[1]{\left\lVert#1\right\rVert}
\newcommand{\vect}[1]{\mathbf{#1}}

\newcommand{\vm}{\vect{m}}

\newcommand{\vx}{\vect{x}}

\newcommand{\vb}{\vect{b}}

\newcommand{\vw}{\vect{w}}

\newcommand{\vq}{\vect{q}}
\newcommand{\vl}{\boldsymbol{\ell}}
\newcommand{\vhw}{\hat{\vw}}
\newcommand{\vbw}{\bar{\vw}}

\newcommand{\R}{\mathbb{R}}

\usepackage{amsmath}
\usepackage{amssymb}
\usepackage{mathtools}
\usepackage{amsthm}
\usepackage{algorithm}
\usepackage{algorithmic}
\usepackage{tikz}
\usetikzlibrary{positioning,calc,shapes.geometric,arrows.meta,decorations,backgrounds,fit}
\usepackage{caption}
\usepackage{bbm}

\theoremstyle{plain}
\newtheorem{theorem}{Theorem}[section]

\newtheorem{lemma}[theorem]{Lemma}

\theoremstyle{definition}

\theoremstyle{remark}

\newcommand{\methodname}{GSQ}

\renewcommand{\paragraph}[1]{\textbf{#1}}
\title{\methodname{}: Highly-Accurate Low-Precision Scalar Quantization for LLMs via Gumbel-Softmax Sampling}

\author{%
  Alireza Dadgarnia\stepcounter{footnote}\thanks{Corresponding authors: alirezadadgarnia1378@gmail.com, dan.alistarh@ist.ac.at} \\
  ISTA \\
  \And
  Soroush Tabesh \\
  ISTA \\
  \And
  Mahdi Nikdan \\
  ISTA \\
  \And
  Michael Helcig \\
  ETH Zürich \\
  \And
  Eldar Kurtić \\
  ISTA \& Red Hat AI \\
  \And
  Maximilian Kleinegger \\
  TU Wien \& ISTA \\
  \And
  Dan Alistarh\footnotemark[2] \\
  ISTA \& Red Hat AI \\
}

\begin{document}

\vspace{-1em}
\maketitle

\vspace{-1em}
\begin{abstract}
Quantization has become a standard tool for efficient LLM deployment, especially for local inference, where models are now routinely served at $2$--$3$ bits per parameter. The state of the art is currently split into simple scalar quantization techniques, such as GPTQ or AWQ, which are widely deployed but plateau in accuracy at 3--4 bits per parameter (bpp), and ``second-generation'' vector- or trellis-quantized methods, such as QTIP, GPTVQ and AQLM, which push the accuracy frontier but are notoriously hard to implement and to scale. In this paper, we ask whether this gap is fundamental, or whether a carefully optimized \textit{scalar} quantizer can recover most of it. We answer in the affirmative, by introducing \methodname{} (Gumbel-Softmax Quantization), a post-training scalar quantization method which jointly learns the per-coordinate grid assignments and the per-group scales using a Gumbel-Softmax relaxation of the discrete grid. \methodname{} matches the cardinality of the relaxation to the small number of levels available in the target bit-width regime (e.g., 3--8 levels for ternary and 3 bpp, respectively), making optimization tractable. Practically, on the standard Llama-3.1-8B/70B-Instruct models, \methodname{} closes most of the gap between scalar quantization and the QTIP frontier at 2 and 3 bits, while using a symmetric scalar grid with group-wise quantization, and thus remains compatible with existing scalar inference kernels. We further show that the same discrete-assignment optimization can be applied to practical GGUF K-Quant checkpoints: starting from publicly released GGUF models, \methodname{} improves accuracy while projecting the result back into the same deployment format. Finally, \methodname{} scales to trillion-scale Mixture-of-Experts models such as Kimi-K2.5, where vector-quantized methods are difficult to apply. The source code is publicly available at \url{https://github.com/IST-DASLab/GSQ}.
\end{abstract}

\vspace{-1em}
\section{Introduction}
\vspace{-0.5em}

The memory and bandwidth costs of LLM inference have made quantization a standard approach for efficient deployment. Among the many quantization directions studied~\citep{gptq, awq, llmint8, spqr, owq, quarot, spinquant, flatquant, smoothquant, quip, quipsharp, qtip, aqlm, gptvq, chen2025efficientqat}, \textit{weight-only} quantization has emerged as the standard for \emph{local} deployment, where the bottleneck is memory rather than compute, and where engines such as llama.cpp~\citep{llamacpp} and Ollama~\citep{ollama} have made models broadly accessible. It is now common to have versions of large open models at around 2--3 bits per parameter, and a whole open-source ecosystem of ``quants'' has emerged around repositories such as Hugging Face~\citep{huggingface, unsloth_2025_kimi_k25}. In practice, many of these models are distributed in deployment-oriented formats such as GGUF K-Quants, which are directly consumed by local inference engines and therefore impose strong compatibility constraints on any improved quantization method.

\paragraph{Existing techniques.}
Broadly, weight quantization techniques can be viewed as two successive  ``waves''. The \emph{first wave} investigated scalar (1D) quantization methods, with llama.cpp~\citep{llamacpp}, bitsandbytes~\citep{bitsandbytes, llmint8}, GPTQ~\citep{gptq}, and AWQ~\citep{awq} being among the most popular methods. These approaches round each weight independently to a small uniform grid, are simple to implement, and benefit from highly optimized unpacking kernels; as a result, they enjoy broad adoption. Their main limitation is \textit{accuracy}: scalar quantization has hit an ``error wall'' around 3--4 bits per parameter (bpp), below which output quality degrades~\citep{gptq, chen2025efficientqat}. The \emph{second wave}, by contrast, has focused on more expressive vector-quantized or trellis representations, including AQLM~\citep{aqlm}, QuIP\#~\citep{quipsharp}, QTIP~\citep{qtip} with its implementation exllamav3~\citep{exllamav3}, and GPTVQ~\citep{gptvq}. By minimizing reconstruction MSE jointly over \textit{groups} of weights, these methods substantially reduce accuracy loss at 2--3 bits per parameter; yet, unfortunately, the resulting representations are considerably harder to implement and scale. Specifically,~\citet{qtip} observed that, although VQ and trellis methods yield major memory savings, they lead to only small decoding speedups vs BF16 due to format complexity. 

We are thus left with a clear gap between accuracy and practicality. The question we ask is:
\emph{Can scalar quantization bridge most of the accuracy gap to complex vector- or trellis-based techniques, while remaining a drop-in replacement for existing popular formats?} 

\paragraph{Our approach.}
We answer this question affirmatively, by proposing \methodname{} (Gumbel-Softmax Quantization), a post-training quantization method which closes most of the gap to second-wave techniques while preserving deployment-friendly scalar structure. In its main form, \methodname{} produces symmetric, group-wise, $b$-bit weights drawn from a small uniform grid, and is therefore directly compatible with existing scalar inference kernels. Yet, on the accuracy side, it substantially improves over previous scalar-focused methods such as GPTQ~\citep{gptq}, AWQ~\citep{awq}, QuIP~\citep{quip}, and EfficientQAT~\citep{chen2025efficientqat} at 2 and 3 bits, where it recovers the majority of the gap to the strongest existing baselines. Beyond this standard scalar setting, we also show that the same local discrete-assignment view can be used to improve existing GGUF K-Quant checkpoints: \methodname{} starts from a  GGUF model, optimizes its assignments, and projects the result back into the same K-Quant format, preserving compatibility with GGUF inference stacks. \methodname{} also scales to very large Mixture-of-Experts (MoE) models, including the trillion-parameter-scale Kimi-K2~\citep{team2025kimi}, where second-wave methods have so far not been applied.

\paragraph{Method overview.}
The key idea is that we want to reformulate layer-wise reconstruction as a \emph{differentiable} discrete-assignment problem. For each weight coordinate, we introduce a small set of trainable logits over the candidate grid points, and obtain a soft quantized weight via Gumbel-Softmax sampling~\citep{maddison2016concrete, jang2016categorical}. The resulting reconstruction loss is fully differentiable in both the per-group scales and the discrete assignments, and can be optimized jointly via gradient-based methods. Sampling allows competing assignments to be explored during
optimization, while annealing makes each relaxed sample increasingly
discrete. The final model uses deterministic assignments obtained
from the learned logits. Our ablations show that stochastic
assignment optimization improves hard reconstruction error over
deterministic softmax under the tested schedule
(Appendix~\ref{app:softmax_vs_gumbel}).

One key observation about grid size is that Gumbel-Softmax relaxation is a natural fit for low-bit scalar quantization. In the regimes we are interested in (e.g. ternary and 2-bit), the cardinality of the per-coordinate grid is small, so a Gumbel-Softmax distribution over the entire grid introduces only a few logits per weight and can be optimized end-to-end at LLM scale. At higher bit-widths (e.g. 3-4 bpp), where the grid grows exponentially, we replace the global relaxation with a \emph{new  local-shift formulation} in which only a small number of nearest grid points around the assignment are considered, keeping overheads small. This combination allows \methodname{} to operate uniformly across the entire low- to mid-bit range. 

\paragraph{Accuracy results.}
Empirically, \methodname{} sets a new state of the art for scalar quantization at low bit-widths. For standard benchmark experiments on Llama-3.1-8B-Instruct and Llama-3.1-70B-Instruct, at 2 bits per parameter, \methodname{} improves average zero-shot accuracy by a remarkable 4.76 and 4.14 points, respectively over the best scalar baseline (EfficientQAT), and trails QTIP, the strongest and most complex prior method, by only 1.33 and 1.68 points, respectively. At 3 bits, the picture is similar: \methodname{} matches or surpasses all scalar baselines, and is essentially on par with QTIP on the 70B model. Notably, these results are obtained with \emph{symmetric} group-wise quantization, without any zero-point parameters, in contrast to most baselines, which is direct evidence that the gains come from better optimization of the discrete assignments rather than from a more flexible quantizer. On the same models, ternary (1.58-bit) \methodname{} already exceeds all scalar baselines even if they are run at higher 2-bit precision. Furthermore, because \methodname{} produces standard scalar layers, it naturally supports non-uniform bit allocation across layers. Using RCO~\citep{helcig2026rco} to search per-layer bit assignments on Llama-3.1-8B-Instruct and Llama-3.1-70B-Instruct, mixed 2/3-bit configurations at 2.37 and 2.62 average bits per parameter retain most of the 3-bit accuracy while substantially reducing model size. We provide speedup results via Humming kernels~\citep{humming2025}.

We further evaluate \methodname{} in the popular GGUF K-Quant setting. Starting from  Unsloth GGUF checkpoints for Qwen3-8B, \methodname{} improves both Q3\_K\_M and Q2\_K models while preserving the same output format. The gains are especially large in the aggressive Q2\_K setting, where the average score across AIME25, GPQA Diamond, and MMLU-Pro improves from 50.03 to 56.28.

Since \methodname{} only requires per-coordinate discrete optimization and per-group scales, its memory footprint is close to that of a standard scalar PTQ approach. This allows us to directly apply \methodname{} to massive Mixture-of-Experts models such as Kimi-K2/2.5, where the codebook training and per-block updates become prohibitively expensive. To our knowledge, \methodname{} is the first method to obtain low-bit, close-to-lossless quantization of trillion-parameter MoE models using a fully scalar, kernel-compatible format.

\vspace{-1em}
\section{Related work}
\vspace{-0.5em}
\label{sec:related}

Early post-training quantization (PTQ) methods such as 
LLM.int8() \citep{llmint8} showed that quantizing 99.9\% of features to INT8 while keeping the outliers in 16-bit achieves significant memory and runtime improvements. \citet{gptq} introduced GPTQ, which uses second-order information to minimize layer-wise quantization error based on the Optimal Brain Surgeon framework \citep{obs}, while AWQ \citep{awq} used activation statistics to identify and protect a small set of weights. Outlier-aware formats \citep{spqr, owq} suggest keeping a small set of outlier weights in high-precision. More recently, rotation-based methods have emerged, enabling near-lossless 4-bit quantization of weight, activations, and KV cache \citep{quarot, spinquant, flatquant}. 
As discussed, vector/codebook methods target 2-bit compression or less while retaining full-precision execution, where QuIP \citep{quip}, QuIP\# \citep{quipsharp}, QTIP \citep{qtip}, AQLM \citep{aqlm}, and PV Tuning \citep{malinovskii2024pv} are strong baselines.

Quantization-aware training (QAT) fine-tunes the model under simulated low-precision arithmetic. EfficientQAT \citep{chen2025efficientqat} makes QAT practical for scalar quantization of LLMs by combining block-wise training of model parameters with a final end-to-end optimization of quantization parameters. Early work on training binary networks \citep{binaryconnect, xnornet} established that 1-bit weights are possible for general deep neural networks. For LLMs, BitNet \citep{bitnet} and follow-up work argue that training in ternary (1.58-bit) can be competitive with full precision. TernaryLLM \cite{ternaryllm} uses trainable scale and zero-point parameters along with a specialized information-theoretic knowledge distillation objective. In the post-training settings, PT$^2$-LLM \citep{pt2llm} enables ternary quantization via iteratively alternating between refining the grid and rounding. Tequila \citep{tequila} reactivates deadzone-trapped weights by re-introducing them as dynamic bias parameters. PTQTP \citep{ptqtp} decomposes the weights into two trit-planes, achieving multiplication-free additive inference. PT-BitNet \citep{ptbitnet} first transforms the weights to make them quantizaition-friendly, then quantizes each weight block separately. For binary PTQ, BiLLM \citep{billm} compresses outlier weights by a binary residual approximation approach, while simply binarizing the remaining weights. DB-LLM \citep{dbllm} decomposes its 2-bit budget into two independent binaries. ARB-LLM \citep{arbllm} present alternating refined binarization to progressively update the binary parameters. PB-LLM \citep{pbllm} simply keeps the salient weights in high-precision, while binarizing the rest. PTQ1.61 \citep{ptq161} takes a similar approach, except that the salient weights are structured and are quantized to 4-bits. STBLLM \citep{stbllm} combines N:M sparsity with binarization of non-pruned weights and provides system support for this format.

MoE quantization must preserve both expert and router quality; QMoE \citep{qmoe} was the first to enable sub-1-bit compression at trillion-parameter scale with custom on-the-fly decoding kernels. MoQa \citep{moqa} assigns different bit-width to each expert based on their sensitivity and distribution of tokens. MoPEQ \citep{mopeq} replaces the criteria with the more rigorous Hessian trace approximation. EAQuant \citep{eaquant} introduces expert smoothing to suppress activation outliers, aligns the logit distribution of the router to preserve expert selection, and balances calibration data across experts. Similarly, ExpertQuant \citep{expertquant} uses a Jaccard loss to ensure the top-k selected experts remain unchanged, while MoEQuant \citep{moequant} specifically addresses the data imbalance problem. Additionally, EAC-MoE \citep{eacmoe} not only calibrates the routers to preserve expert selection, but also suggests pruning less frequently used experts.

Because quantization and sparsity introduce discrete choices (rounding, masks), many approaches use differentiable proxies. The straight-through estimator (STE)~\citep{bengio2013ste} keeps a hard assignment in the forward pass and treats rounding as the identity in the backward pass; LSQ \citep{lsq} uses it to learn quantizer step sizes. The rotation trick~\citep{fifty2025rotation} propagates gradients through vector quantization using a rotation and rescaling transformation that encodes the relative angle and magnitude between an input vector and its selected codeword. DiffQ \citep{diffq} uses pseudo quantization noise to optimize bit allocation. For post-training rounding, AdaRound~\citep{adaround} learns whether each weight rounds up or down through a deterministic sigmoid relaxation. The Gumbel-Softmax relaxation has been applied to neural architecture search \citep{herrmann2020channel}, and MaskLLM \citep{maskllm} uses it to learn semi-structured N:M masks end-to-end.

Relaxed Quantization (RQ)~\citep{louizos2019relaxed} applies
Gumbel-Softmax relaxations to weight and activation
quantization. RQ derives categorical probabilities by integrating
a continuous logistic distribution over quantization bins, and
jointly optimizes the underlying continuous values and quantizer
parameters. It also introduces a local grid to reduce the cost
of sampling at higher bit-widths. \methodname{} instead directly
learns assignment logits over scalar grid points, or over a fixed
set of index shifts around an initial quantized assignment, together
with per-group scales.

\section{The Gumbel-Softmax quantization (\methodname{}) method}
\label{sec:method}
\paragraph{Notation and Gumbel-Softmax Sampling.}
Let $f(\cdot; \vw)$ denote a function parameterized by weights $\vw \in \R^{d}$ that we aim to compress, e.g. a single linear layer, a sub-module (such as a Transformer block), or an entire  network. Given calibration input data $\vx$, the objective is to find a parameterization $\vhw$ that satisfies a constraint set $\mathcal{C}$ while minimizing the output reconstruction error, measured e.g. in the Frobenius norm $\norm{\cdot}_F$: 
\begin{equation}
\vhw = \argmin_{\vbw}~ \norm{f(\vx;\vbw) - f(\vx;\vw)}_F^2 \quad \text{s.t.} \quad \vbw \in \mathcal{C}.
\label{eq:objective}
\end{equation}

Optimizing the objective in Equation \ref{eq:objective} becomes challenging when the constraint set $\mathcal{C}$ includes discrete components, since standard gradient-based methods are not directly applicable. This difficulty arises naturally in model compression: quantization requires mapping weights to a discrete grid, while sparsity needs setting a specific subset of weights to zero. To address this challenge, our method leverages Gumbel-Softmax sampling \citep{maddison2016concrete, jang2016categorical} to make the discrete selection process differentiable.

For candidates $\mathcal{D}=\{d_1,\ldots,d_K\}$, let $\ell_i$
denote the trainable assignment logits. Independent perturbations
$g_i\sim\mathrm{Gumbel}(0,1)$ provide an exact categorical sampler
through the Gumbel--Max identity:
\[
\Pr\!\left[\argmax_j\{\kappa\ell_j+g_j\}=i\right]
=\frac{\exp(\kappa\ell_i)}
       {\sum_j\exp(\kappa\ell_j)}.
\]
Algorithm~\ref{alg:gumbel-softmax} replaces this hard selection
with the differentiable relaxation
$p_i=\mathrm{softmax}_i((\kappa\boldsymbol{\ell}+\mathbf{g})/\tau)$
and returns $\tilde d=\sum_i p_i d_i$.
The logit scale $\kappa$ controls the categorical sampling
probabilities, whereas the temperature $\tau$ controls the softness
of each relaxed sample. For fixed logits and $\kappa$, taking
$\tau\to0$ recovers a categorical sample. For deployment, we remove
the perturbations and select the candidate with the largest
learned logit.

The perturbations allow alternative assignments to be sampled
throughout optimization. In contrast, a deterministic relaxation
can concentrate around early choices and provide weak gradients
for reconsidering them. At finite temperature, \methodname{} uses
soft weights in the forward pass and differentiates through the
relaxation, whereas the STE alternatives in our ablation use hard
forward assignments
(Appendix~\ref{app:softmax_vs_gumbel}). Variance-matched Gaussian,
logistic, and uniform perturbations perform similarly to Gumbel
perturbations in a separate study
(Appendix~\ref{app:noise_ablation}). We therefore choose Gumbel
noise for its exact categorical-sampling interpretation, without
claiming empirical superiority over other noise families.

This optimization introduces $|\mathcal{D}|$ learnable logits. In the special case where $\mathcal{D}$ contains only two elements, instead of introducing two separate logits, we use a single logit $\ell$ and assign $-\ell$ as the logit for the other element. The probability assigned to the candidate with logit $+\ell$ is $\operatorname{sigmoid}((2\kappa\ell+g_+-g_-)/\tau)$. This halves the number of trainable parameters in the binary case and substantially lowers the memory overhead.

\subsection{The ternary quantization case}
We begin by describing how Gumbel-Softmax sampling is used to compress the model parameters into a ternary quantization format. Specifically, we impose the following constraint in Objective~\ref{eq:objective}:
\begin{align}
\mathcal{C}_\text{ternary} &= \big\{ \vbw \mid \vbw = s \cdot \vm \odot \vb; s \in \R, \vm \in \{0,1\}^d, \vb \in \{-1,1\}^d \big\}.
\end{align}
Under this formulation, a ternary-quantized vector is parameterized by three components: a binary mask $\vm$ indicating which entries are zero, a binary sign vector $\vb$ specifying whether each nonzero entry is $-1$ or $+1$, and a scaling factor $s$. This parameterization introduces $2d$ binary decisions, which we relax using $2d$ instances of binary Gumbel-Softmax sampling. Concretely, we jointly optimize the scale $s$, the mask logits $\vl^{(m)} \in \R^d$, and the sign logits $\vl^{(b)} \in \R^d$. At each training step, $\vm$ and $\vb$ are obtained by applying Gumbel-Softmax sampling to $\vl^{(m)}$ and $\vl^{(b)}$, respectively. The full procedure is provided in Algorithm \ref{alg:ternary}. Although the formulation above assumes a single shared scale value (i.e., symmetric global quantization), the same framework extends naturally to asymmetric and/or group-wise quantization with minor modifications.

\begin{figure}[t]
\centering
\captionsetup{type=algorithm}

\begin{minipage}[t]{0.47\linewidth}
\centering
\captionof{algorithm}{Gumbel-Softmax (GS) Sampling}
\label{alg:gumbel-softmax}
\vspace{0.25em}
\footnotesize
\begin{algorithmic}[1]
  \REQUIRE Finite set $\mathcal{D}=\{d_1,\ldots,d_n\}$
  \REQUIRE Logits $\ell_1,\ldots,\ell_n$
  \REQUIRE Temperature $\tau>0$, noise scale $\kappa>0$
  \ENSURE Soft sample $\tilde d$

  \FOR{$i = 1$ \textbf{to} $n$}
    \STATE Draw $g_i \sim \mathrm{Gumbel}(0,1)$
  \ENDFOR

  \FOR{$i = 1$ \textbf{to} $n$}
    \STATE
    $\displaystyle
    p_i \leftarrow
    \frac{
      \exp\!\bigl((\kappa \ell_i + g_i)/\tau\bigr)
    }{
      \sum_{j=1}^{n}
      \exp\!\bigl((\kappa \ell_j + g_j)/\tau\bigr)
    }$
  \ENDFOR

  \STATE \textbf{return}
  $\displaystyle \tilde d \leftarrow \sum_{i=1}^{n} p_i d_i$
\end{algorithmic}
\end{minipage}
\hfill
\begin{minipage}[t]{0.50\linewidth}
\centering
\captionof{algorithm}{Ternary \methodname{}}
\label{alg:ternary}
\vspace{0.25em}
\footnotesize
\begin{algorithmic}[1]
  \REQUIRE Weights $\vw \in \mathbb{R}^d$, calibration data $\vx$
  \REQUIRE Schedules $\{\tau_t\}_{t=1}^T$, $\{\kappa_t\}_{t=1}^T$
  \ENSURE Ternary weights $\vhw \in \{-s,0,s\}^d$

  \STATE Initialize scale $s \in \mathbb{R}$, logits $\vl^{(m)},\vl^{(b)} \in \mathbb{R}^{d}$

  \FOR{$t = 1$ \textbf{to} $T$}
    \FOR{$i = 1$ \textbf{to} $d$}
      \STATE
      $\displaystyle
      \tilde{\vm}_i \leftarrow
      \mathrm{GS}
      \bigl(
        \{0,1\},
        \{-\vl^{(m)}_i,\vl^{(m)}_i\},
        \tau_t,\kappa_t
      \bigr)$
      \STATE
      $\displaystyle
      \tilde{\vb}_i \leftarrow
      \mathrm{GS}
      \bigl(
        \{-1,1\},
        \{-\vl^{(b)}_i,\vl^{(b)}_i\},
        \tau_t,\kappa_t
      \bigr)$
    \ENDFOR

    \STATE $\vbw \leftarrow s \cdot \tilde{\vm} \odot \tilde{\vb}$
    \STATE $\mathcal{L} \leftarrow
      \norm{f(\vx;\vbw)-f(\vx;\vw)}_F^2$
    \STATE Update $s,\vl^{(m)},\vl^{(b)}$ using $\nabla \mathcal{L}$
  \ENDFOR

  \STATE $\hat{\vm}_i \leftarrow \mathbbm{1}\{\vl^{(m)}_i \ge 0\}$,  $\quad \hat{\vb}_i \leftarrow 2\mathbbm{1}\{\vl^{(b)}_i \ge 0\}-1$
  \STATE \textbf{return}
  $\vhw \leftarrow s \cdot \hat{\vm} \odot \hat{\vb}$
\end{algorithmic}
\end{minipage}

\vspace{0.25em}
\caption*{Algorithms 1 and 2: The Gumbel-Softmax relaxation, denoted by $\mathrm{GS}$ in Algorithm~\ref{alg:ternary}, and its use inside ternary \methodname{} to learn differentiable mask and sign variables before hard thresholding.}
\label{fig:gumbel-ternary-algorithms}
\end{figure}

\paragraph{Initialization.}
Instead of initializing the logits randomly, we warm-start from the GPTQ ternary solution $\vq_\text{GPTQ} \in \{-1,0,1\}^d$ \citep{gptq}. Recall that the mask logit $\vl^{(m)}_i$ controls whether weight $i$ is nonzero: a positive logit favors $\vm_i=1$ (active), while a negative logit favors $\vm_i=0$ (pruned).
Similarly, the sign logit $\vl^{(b)}_i$ controls the sign of the nonzero weight: a positive logit favors $\vb_i=+1$ and a negative logit favors $\vb_i=-1$. We therefore initialize each logit to reflect the corresponding GPTQ decision:
\begin{equation}
(\vl^{(m)}_\text{GPTQ})_i =
\begin{cases}
+1.0, & \text{if } (\vq_\text{GPTQ})_i \neq 0,\\
-1.0, & \text{if } (\vq_\text{GPTQ})_i = 0,
\end{cases}
\qquad
(\vl^{(b)}_\text{GPTQ})_i =
\begin{cases}
+1.0, & \text{if } (\vq_\text{GPTQ})_i = +1,\\
-1.0, & \text{if } (\vq_\text{GPTQ})_i = -1,\\
\phantom{+}0.0, & \text{if } (\vq_\text{GPTQ})_i = 0.
\end{cases}
\label{eq:ternary_init}
\end{equation}
When $(\vq_\text{GPTQ})_i = 0$, the sign logit is initialized to $0$, since the sign is irrelevant at initialization but may become active during subsequent optimization if the mask flips to nonzero.
To prevent getting stuck, we inject isotropic Gaussian noise into the logits before training. We initialize the mask and sign logits as
\begin{equation}
\vl = \sigma_{\mathrm{init}}\Big(\boldsymbol{\epsilon} +
\alpha\,\vl_\text{GPTQ}\Big),
\qquad
\boldsymbol{\epsilon}\sim\mathcal{N}(\mathbf{0},\mathbf{I}),
\label{eq:logit_init}
\end{equation}
where $\alpha \in \R$ controls the strength of the GPTQ warm-start relative to the injected noise, and $\sigma_{\mathrm{init}} \in \R$ sets the overall scale of the logits. We also initialize the quantization scale value $s$ to the scale computed by GPTQ.

\subsection{General scalar quantization}
We now describe how \methodname{} extends to general scalar quantization. Suppose the goal is to quantize the model parameters to $b$ bits using symmetric quantization with a single shared scale factor. In this setting, the constraint set $\mathcal{C}$ in Objective~\ref{eq:objective} can be written as
\begin{equation}
\mathcal{C}_\text{$b$-bit} = \left\{ \vbw \middle| \vbw = s \cdot \vq;\ s \in \mathbb{R},\ \vq \in \mathcal{G}_b^d \right\},
\end{equation}
where $\mathcal{G}_b$ denotes the ordered quantization grid, with cardinality $\lvert \mathcal{G}_b \rvert = 2^b$, specifying the set of values that each quantized parameter may take. This formulation imposes no structural restrictions on the grid and therefore accommodates both uniform and non-uniform quantization schemes.

To enable gradient-based optimization, we apply Gumbel--Softmax sampling independently to each of the $d$ coordinates. Each such instance introduces $2^b$ trainable logits, i.e., the logits can be concatenated into $\vl \in \R^{d \times 2^b}$. With this relaxation, we jointly optimize the logits $\vl$ and the scale parameter $s$.

\paragraph{The 2-bit case.} As a direct application, consider 2-bit uniform quantization with $\mathcal{G}_2 = \{-2, -1, 0, 1\}$. In this case, each coordinate is associated with 4 trainable logits for Gumbel--Softmax sampling. Together with the shared scale parameter $s$, this yields a total of $4d + 1$ trainable parameters. 
Although the grid itself is skewed toward negative values, we allow the scale $s$ to take negative values as well, thereby removing any inherent bias toward either side.

\paragraph{Higher bit-widths.} As the bit-width $b$ increases, the number of trainable logits, and consequently the required memory, grows exponentially in bpp. To address this, for $b > 2$, we use a local shift-based formulation, explained below and summarized in Figure~\ref{fig:local-shift}.
The key observation is that, each coordinate typically remains close to its initialized quantized value, and large jumps across the quantization grid rarely happen (see Appendix~\ref{app:ablation_local_shift}). Instead of assigning a logit to every value in $\mathcal{G}_b$, we only learn a small relative shifts.
Specifically, suppose we are given an initialized quantized vector $\vq^0 \in \mathcal{G}_b^d$. For each coordinate $i \in \{1,\dots,d\}$, let $j_i^0 \in \{1,\dots,2^b\}$ denote the index of the initialized grid point, i.e.,
$q_i^0 = (\mathcal{G}_b)_{j_i^0}.$

Instead of introducing $2^b$ logits for coordinate $i$, we introduce only 5 logits corresponding to a discrete shift $\delta_i \in \{-2,-1,0,1,2\}$. Let $\vl_\delta \in \mathbb{R}^{d \times 5}$ denote the corresponding trainable logits. At each training step, $\delta_i$ is obtained by applying Gumbel-Softmax sampling to the $i$-th row of $\vl_\delta$. The resulting grid index is then 
$j_i = \mathrm{clip}(j_i^0 + \delta_i, 1, 2^b),$ 
where $\mathrm{clip}(x,a,b) = \min\{\max\{x,a\}, b\}$ clips the value into the valid range. The final quantized value is
$q_i = (\mathcal{G}_b)_{j_i}.$

Equivalently, the constraint set becomes
\begin{equation}
\mathcal{C}^{\mathrm{shift}}_\text{$b$-bit}
=
\left\{
\vbw \middle|
\vbw = s \cdot \vq,\;
s \in \mathbb{R},\;
q_i = (\mathcal{G}_b)_{\mathrm{clip}(j_i^0 + \delta_i,\,1,\,2^b)},\;
\delta_i \in \{-2,-1,0,1,2\}
\right\}.
\end{equation}
Under this parameterization, each coordinate requires only 5 trainable logits rather than $2^b$. Therefore, the total number of trainable parameters is reduced from $d \times 2^b + 1$ to $5d + 1$, making higher-bit optimization practical while still allowing each coordinate to move to nearby grid values.

\paragraph{Initialization.}
We distinguish the initial quantized assignment from its
representation in the relaxed optimization. GPTQ supplies the
initial assignments and scales. We initialize the assignment logits
with a Gaussian-like prior over candidate distance, centered on the
selected grid value for 2-bit quantization and on zero shift for
higher bit-widths (Appendix~\ref{app:initialization}).
This prior favors the initial assignment while assigning greater
probability to nearby alternatives than to distant ones.

RTN prior assignments perform nearly as well as GPTQ, whereas random assignments fail (Appendix~\ref{app:init_ablation}). With GPTQ assignments fixed, the Gaussian-like prior substantially outperforms an initialization that boosts only the selected candidate's logit. Thus, \methodname{} needs a meaningful initial assignment, though not GPTQ specifically, and benefits from a distance-aware logit initialization.

\subsection{Implementation details}

\paragraph{Objective.}
Unlike most PTQ methods, \methodname{} is not tied to a layerwise quadratic objective. In principle, it can optimize richer objectives directly over the quantized parameters, such as block-level reconstruction losses or model-level task-aware losses. This flexibility, however, comes at the cost of additional memory, since \methodname{} introduces auxiliary trainable logits whose footprint is typically 2-5$\times$ that of the weights being quantized. As a result, jointly optimizing the entire model is prohibitively expensive for large Transformers and MoEs.

In practice, we adopt a combination of objectives during optimization. For example, in the ternary quantization setting, where the GPTQ initialization is particularly weak, we first warm up the logits (initialized from GPTQ) using a cheaper layerwise quadratic objective applied independently to each linear layer. This is then followed by a blockwise or expertwise optimization stage. The exact procedure depends on the model and setting, and is described in Section~\ref{sec:experiments}.

We note that this differs from approaches that first quantize each layer independently, and afterwards perform a limited block-level tuning over a small subset of continuous parameters such as the quantization scales \citep{quipsharp, aqlm}. In \methodname{}, the discrete assignments and their associated continuous parameters are optimized jointly.

\paragraph{Optimizer.} The Gumbel-Softmax relaxation can enter a saturated regime in which the relaxed categorical distribution becomes nearly one-hot, for example due to temperature annealing or growing logit gaps; see Appendix~\ref{app:vanishing_logit_grads}. ~\citep{maskllm} mitigate this through problem-specific regularization and by increasing the $\epsilon$ hyperparameter in AdamW (e.g., to $10^{-5}$). When gradients vanish, AdamW effectively stalls. We address this differently, by using Lion~\citep{lion}, which is much less sensitive to vanishing gradients.

\paragraph{Gradient accumulation.}
 In our setting, gradient accumulation plays an additional role, beyond memory savings: because we \emph{resample} the Gumbel noise independently for each forward pass, averaging gradients across micro-batches reduces sampling variance. This leads to noticeably more stable optimization.

\vspace{-1.0em}
\section{Experiments}
\vspace{-0.5em}
\label{sec:experiments}

\subsection{Experimental setup}
\label{sec:exp_setup}

\paragraph{Models.}
We evaluate on two standard dense models, Llama-3.1-8B-Instruct and Llama-3.1-70B-Instruct~\citep{grattafiori2024llama}, as well as the MoE model Kimi-K2.5~\citep{team2026kimi}. For the Llama models, we quantize all non-embedding and non-head linear layers, with one exception: in the 8B model, we find the \texttt{down\_proj} of the second block to be unstable under high compression and leave it in full precision. For Kimi-K2.5, we quantize only the non-shared expert weights while leaving the shared experts untouched; we also skip the vision-related components and only evaluate the language model. Finally, to evaluate compatibility with existing local-deployment formats, we also apply \methodname{} to publicly available GGUF K-Quant checkpoints of Qwen3-8B, using the released quantized model as initialization and preserving the same GGUF format after optimization.

\paragraph{Quantization configuration.}
For scalar experiments, we focus on 2-bit and 3-bit weight-only quantization with a group size of 128. \methodname{} uses a symmetric scalar quantizer, where each group shares a single scale value. Groups are formed row-wise over consecutive entries, following the standard packing layout used in prior work. We also include brief ternary quantization experiments (i.e., 1.58-bit). In addition, we evaluate \emph{non-uniform} bit allocation, in which different layers are assigned different bit-widths  so as to achieve a fractional average rate such as 2.37 or 2.62 bits per parameter.

\paragraph{Training details.}
We perform block-wise optimization with a Gumbel-Softmax relaxation over the discrete assignments, followed by scale-only fine-tuning only for 2-bit Llama experiments. 
Full hyperparameters and sensitivity results appear in
Appendices~\ref{app:train_hparams}
and~\ref{app:gsq_hparam_sensitivity}.

\paragraph{Within-block staging.}
For dense models, we found that optimizing all quantized layers in a Transformer block jointly under a block reconstruction loss was suboptimal, despite being the natural formulation used in prior work~\citep{aqlm,chen2025efficientqat}. We instead use a staged schedule: query and key projections are first optimized independently with linear reconstruction losses; value and output projections are then optimized jointly under the self-attention output reconstruction loss; and the MLP projections are optimized last under the full block reconstruction loss. After a block is quantized, it is frozen, and subsequent blocks are optimized using inputs from the already-quantized prefix, making the procedure aware of accumulated quantization error. We provide further motivation and details in Appendix~\ref{app:within_block_staging}.

\paragraph{Calibration data and training budget.}
For calibration data, we use FineWeb-Edu~\citep{lozhkov2024fineweb-edu} in Llama experiments, and OpenThoughts \citep{guha2025openthoughtsdatarecipesreasoning} for Kimi K2.5 experiments. Unless otherwise stated, we use 4096 sequences of length 4096. For block-wise training, we run 20 epochs for the Llama models and 10 epochs for Kimi-K2.5. For end-to-end scale-only fine-tuning on the Llama models, we run a single epoch over the same 4096 sequences.

\paragraph{Scalar quantization baselines.}
We include GPTQ~\citep{gptq}, QuIP~\citep{quip}, and EfficientQAT~\citep{chen2025efficientqat} as baselines; these methods are allowed to use asymmetric quantization with per-group zero-points, giving them strictly more representational freedom than \methodname{}. Whenever a released quantized checkpoint is available, we re-evaluate it directly under our evaluation pipeline; otherwise, we run the official codebase with the hyperparameters recommended by the original authors. For GPTQ and QuIP, we use 512 calibration samples as is standard, and for EfficientQAT we follow the authors' suggested setup.

\paragraph{Vector quantization baselines.}
We compare against QTIP~\citep{qtip} and PV-Tuning~\citep{malinovskii2024pv}, which optimizes over an AQLM vector quantized representation, as these are two state-of-the-art methods in the low-bit regime. Since VQ methods are not restricted to a small scalar grid, they are generally more expressive than scalar quantizers at the same bit-width; we therefore view the comparison to VQ as a particularly hard test for \methodname{}.

\subsection{Llama-3.1 results}
\label{sec:exp_llama}

\begin{table*}[t]
\centering
\setlength{\tabcolsep}{3.5pt}
\caption{Zero-shot accuracy on Llama-3.1-8B-Instruct and Llama-3.1-70B-Instruct under ternary, 2-bit, 3-bit, and non-uniform quantization. Bit/param excludes non-quantized tensors.}
\resizebox{\textwidth}{!}{%
\begin{tabular}{l|ccccccc|ccccccc}
\toprule
& \multicolumn{7}{c|}{\textbf{Llama-3.1-8B-Instruct}} & \multicolumn{7}{c}{\textbf{Llama-3.1-70B-Instruct}} \\
\textbf{Method}
& bit/param & ARC-C & ARC-E & Hella. & PIQA & Wino. & Avg.
& bit/param & ARC-C & ARC-E & Hella. & PIQA & Wino. & Avg. \\
\midrule
FP16/BF16               & 16 & 55.12 & 79.63 & 79.16 & 80.85 & 73.80 & 73.71
                         & 16 & 63.48 & 83.92 & 84.58 & 83.95 & 79.01 & 78.99 \\
\midrule \midrule
QTIP                         & 3.00 & 53.92 & 79.42 & 78.30 & 80.25 & 72.14 & 72.81
                         & 3.00& 61.77 & 82.79 & 84.21 & 83.79 & 78.30 & 78.17 \\
\midrule
GPTQ                       & 3.25 & 39.68 & 58.67 & 64.24 & 67.03 & 66.30 & 59.18
                         & 3.25 & 60.75 & 80.64 & 82.37 & \textbf{83.30} & 77.98 & 77.01 \\
QuIP                         & 3.25 & 52.30 & 76.68 & 75.37 & 78.84 & 71.90 & 71.02
                         & 3.25 & \textbf{62.12} & 82.45 & 82.83 & 82.10 & 78.06 & 77.51 \\
EfficientQAT             & 3.25 & \textbf{52.99} & \textbf{78.91} & \textbf{76.85} & 79.76 & 71.59 & 72.02
                         & 3.25 & 61.86 & \textbf{83.88} & 82.79 & 82.48 & 76.01 & 77.40 \\
\textbf{GSQ (ours)}   & 3.13 & \textbf{52.99} & 78.37 & 76.66 & \textbf{80.03} & \textbf{73.56} & \textbf{72.32}
                         & 3.13 & \textbf{62.12} & 82.83 & \textbf{83.30} & 82.75 & \textbf{78.93} & \textbf{77.99} \\
\midrule \midrule
\textbf{GSQ (ours)}   & 2.62\textsuperscript{NU} & 49.23 & 74.79 & 74.83 & 79.00 & 70.24 & 69.62
                         & 2.62\textsuperscript{NU} & 59.60 & 81.40 & 82.90 & 82.90 & 80.40 & 77.50 \\
\textbf{GSQ (ours)}   & 2.37\textsuperscript{NU} & 50.26 & 74.41 & 74.81 & 78.13 & 69.61 & 69.44
                         & 2.37\textsuperscript{NU} & 60.20 & 81.40 & 82.60 & 82.40 & 80.60 & 77.40 \\
\midrule \midrule
QTIP                        & 2.00 & 50.68 & 75.42 & 75.02 & 78.18 & 70.09 & 69.88
                         & 2.00 & 61.69 & 81.69 & 82.95 & 82.43 & 77.51 & 77.25 \\
PV-Tuning                & 2.27 & 50.26 & 73.91 & 75.28 & 79.16 & 70.56 & 69.83
                         & 2.07 & 58.62 & 80.72 & 82.72 & 81.56 & 77.74 & 76.27 \\
\midrule
GPTQ                       & 2.25 & 24.91 & 27.61 & 30.50 & 52.83 & 51.78 & 37.53
                         & 2.25 & 38.23 & 58.63 & 60.11 & 72.80 & 57.14 & 57.38 \\
QuIP                         & 2.25 & 23.72 & 28.96 & 38.65 & 53.86 & 50.83 & 39.20
                         & 2.25 & 39.51 & 60.44 & 68.95 & 73.61 & 65.35 & 61.57 \\
EfficientQAT             & 2.25 & 43.77 & 67.55 & 68.65 & 74.65 & 64.33 & 63.79
                         & 2.25 & 54.86 & 77.27 & 79.01 & 80.36 & 65.67 & 71.43 \\
\textbf{GSQ (ours)}  & 2.13 & \textbf{48.12} & \textbf{72.35} & \textbf{73.42} & \textbf{78.07} & \textbf{70.80} & \textbf{68.55}
                         & 2.13 & \textbf{58.87} & \textbf{79.55} & \textbf{82.11} & \textbf{81.07} & \textbf{76.24} & \textbf{75.57} \\ 
\midrule \midrule
\textbf{GSQ (ours)}  & 1.71 & \textbf{42.83} & \textbf{67.13} & \textbf{67.91} & \textbf{73.50} & \textbf{65.82} & \textbf{63.44}
                         & -- & -- & -- & -- & -- & -- & -- \\ 
\bottomrule
\end{tabular}%
}
\label{tab:llama_all}
\end{table*}

Table~\ref{tab:llama_all} reports zero-shot accuracy on Llama-3.1-8B-Instruct and Llama-3.1-70B-Instruct. Across both model sizes, \methodname{} is the strongest scalar quantization method in our comparison and remains competitive with recent VQ methods.

At 2 bits, \methodname{} substantially outperforms GPTQ, QuIP, and EfficientQAT despite using symmetric quantization without zero-points. On Llama-3.1-70B, it improves over the best scalar baseline by 4.14 average points and narrows the gap to VQ methods to 1.68 points vs. QTIP and 0.70 points vs. PV-Tuning. At 3 bits, \methodname{} again gives the best scalar results and approaches the VQ frontier, showing that the benefit is not limited to the most extreme 2-bit setting. Finally, ternary \methodname{} on Llama-3.1-8B is competitive with, and often exceeds, scalar baselines evaluated at higher 2-bit precision.

\begin{wraptable}{r}{0.5\textwidth}
\vspace{-1.0em}
\centering
\small
\setlength{\tabcolsep}{4pt}
\caption{End-to-end inference speedup for \methodname{}-quantized Llama-3.1-70B-Instruct on NVIDIA L40s GPUs (vLLM + Humming kernels). TPS/GPU = output tokens per second per GPU.}
\scalebox{0.8}{
\begin{tabular}{l|c|cc}
\toprule
\multirow{2}{*}{\textbf{Method}} & \multirow{2}{*}{\textbf{Avg bit/param}} & \textbf{ShareGPT} & \multirow{2}{*}{\textbf{Speedup}} \\
 & & \textbf{TPS/GPU} & \\
\midrule
BF16 (4 GPU)      & 16.00 &  60.3 & 1.00$\times$ \\
\midrule
Uniform 3-bit     &  3.13 & 289.6 & 4.80$\times$ \\
Non-uniform 2.62  &  2.62 & 301.1 & 4.99$\times$ \\
Non-uniform 2.37  &  2.37 & 329.2 & 5.46$\times$ \\
Uniform 2-bit     &  2.13 & 374.0 & 6.20$\times$ \\
\bottomrule
\end{tabular}
}
\label{tab:speedup_70b}
\vspace{-1.0em}
\end{wraptable}
\paragraph{Non-uniform results.}
Because \methodname{} produces standard scalar quantized layers, it naturally supports \emph{non-uniform} bit allocation, in which different layers are assigned different bit-widths (e.g., some layers at 3 bits and others at 2 bits) to achieve a target average rate. This is motivated by the observation that not all layers are equally sensitive to quantization: allocating more bits to sensitive layers and fewer to robust ones can yield a better accuracy-compression trade-off than using a single uniform bit-width everywhere. Per-layer bit assignments are obtained via non-uniform search~\citep{helcig2026rco}. Table~\ref{tab:llama_all} includes non-uniform \methodname{} results on Llama-3.1-70B-Instruct at 2.62 and 2.37 average bits per parameter. These configurations interpolate between the uniform 3-bit and 2-bit operating points. Llama-3.1-8B-Instruct is also evaluated at 2.37 average bits per parameter.

\paragraph{Speedup.}
A key advantage of scalar quantization over vector- or trellis-based methods is that it can directly leverage highly optimized low-precision GEMM kernels, translating memory savings into proportional throughput gains. Table~\ref{tab:speedup_70b} reports end-to-end throughput for \methodname{}-quantized Llama-3.1-70B-Instruct using vLLM~\citep{kwon2023efficient} and Humming kernels~\citep{humming2025} on NVIDIA L40s GPUs. We report output tokens per second per GPU (TPS/GPU) to normalize across tensor-parallelism configurations, using a representative ShareGPT (short, conversational) workload.\footnote{\url{https://huggingface.co/datasets/anon8231489123/ShareGPT_Vicuna_unfiltered}} Uniform 2-bit quantization achieves up to 6.2$\times$ speedup over BF16, while the non-uniform configurations at 2.37 and 2.62 average bits provide 4.99--5.46$\times$ speedup, demonstrating that non-uniform bit allocation offers a practical accuracy--throughput trade-off within a single, kernel-compatible scalar format.

\subsection{Kimi-K2.5 results}
\label{sec:exp_kimi}

We further evaluate \methodname{} on Kimi-K2.5, a 1T-parameter MoE model natively stored in 4-bit format. We quantize only the non-shared expert weights to 2 bits, leaving shared experts and vision components unchanged, and compare against the original model under the same evaluation pipeline.

Table~\ref{tab:kimi_combined}
shows that 2-bit \methodname{} largely preserves Kimi-K2.5's reasoning and coding performance. It remains close on AIME25, improves on LiveCodeBench-v6 and MATH500, but drops on GPQA Diamond. We attribute this pattern partly to calibration data: OpenThoughts is skewed toward mathematics and code, whereas GPQA Diamond emphasizes science-heavy question answering. To further examine this hypothesis, Appendix~\ref{sec:exp_kimi26} reports an additional Kimi-K2.6 experiment using a more diverse calibration mixture. Although this is not a controlled ablation since the base model also changes, the GPQA gap is substantially smaller, indicating the importance of calibration diversity.

Table~\ref{tab:kimi_combined} additionally
reports OpenAI-MRCR results across context lengths up to 256k. \methodname{} improves over the base model up to 32k tokens and remains close at longer ranges, with modest degradation beyond 32k. This suggests that aggressive 2-bit quantization preserves much of the model's long-context capability, although the longest contexts are more sensitive.

Our LiveCodeBench score for the original Kimi-K2.5 model is lower than the model-card number, likely due to differences in evaluation protocol. Since the exact protocol is not publicly specified, we report both base and quantized results under our common pipeline.

Lastly, we evaluate Kimi-K2.5 on agentic benchmarks. Results can be found in Appendix \ref{sec:exp_agentic}.

\begin{table*}[t]
\centering
\small
\setlength{\tabcolsep}{4pt}
\caption{Results on Kimi-K2.5. We compare the original model and its 2-bit \methodname{} version on mathematical reasoning, scientific QA, coding benchmarks, and long-context OpenAI-MRCR across different context lengths. OpenAI-MRCR uses 3 repeats and 4 needles.}
\small
\setlength{\tabcolsep}{2pt}
\resizebox{\textwidth}{!}{
\begin{tabular}{lccccccccccc}
\toprule
\multirow{4}{*}{\textbf{Method}} & \multirow{4}{*}{bpp}
& \multicolumn{4}{c}{\textbf{Reasoning and Coding}}
& \multicolumn{6}{c}{\textbf{OpenAI-MRCR}} \\
\cmidrule(lr){3-6} \cmidrule(lr){7-12}
& & AIME 25 & GPQA Diamond & LiveCodeBench & MATH 500
& \multirow{2}{*}{0--8k}
& \multirow{2}{*}{8k--16k}
& \multirow{2}{*}{16k--32k}
& \multirow{2}{*}{32k--64k}
& \multirow{2}{*}{64k--128k}
& \multirow{2}{*}{128k--256k} \\
& & ($n=10$) & ($n=5$) & ($n=10$) & ($n=5$)
& & & & & & \\
\midrule
Base & 4.5
& \textbf{95.33} & \textbf{89.29} & 61.37 & 96.68
& 95.37 & 77.81 & 69.19 & \textbf{57.10} & \textbf{59.50} & \textbf{46.04} \\
\methodname{} (Ours) & 2.13
& 93.00 & 76.57 & \textbf{69.37} & \textbf{97.32}
& \textbf{97.81} & \textbf{85.75} & \textbf{74.41} & 53.73 & 59.03 & 44.91 \\
\bottomrule
\end{tabular}
}
\label{tab:kimi_combined}
\vspace{-1.0em}
\end{table*}

\subsection{Extension to GGUF K-Quant formats}
\label{sec:exp_gguf}

\begin{wraptable}{r}{0.59\textwidth}
\vspace{-1.0em}
\centering
\setlength{\tabcolsep}{5pt}
\caption{GGUF K-Quant results on Qwen3-8B.}
\small
\scalebox{0.73}{
\begin{tabular}{l|cc|cc|cc|c}
\toprule
\multirow{2}{*}{Method} 
& \multicolumn{2}{c|}{AIME25} 
& \multicolumn{2}{c|}{GPQA Diamond}
& \multicolumn{2}{c|}{MMLU-Pro}
& \multirow{2}{*}{Avg.} \\
& Acc. & Tokens 
& Acc. & Tokens 
& Acc. & Tokens 
& \\
\midrule
BF16
& 64.67 & 0.52M
& 58.38 & 1.48M
& 72.42 & 32.07M
& 65.16 \\
\midrule
Unsloth-Q3\_K\_M
& 56.67 & 0.58M
& 54.65 & 1.67M
& 70.23 & 35.11M
& 60.52 \\
\methodname{}-Q3\_K\_M
& \textbf{59.00} & \textbf{0.55M}
& \textbf{55.15} & \textbf{1.53M}
& \textbf{70.67} & \textbf{31.57M}
& \textbf{61.61} \\
\midrule
Unsloth-Q2\_K
& 38.67 & 0.64M
& 46.46 & 2.46M
& 64.97 & 48.78M
& 50.03 \\
\methodname{}-Q2\_K
& \textbf{52.00} & \textbf{0.58M}
& \textbf{50.91} & \textbf{1.80M}
& \textbf{65.94} & \textbf{39.69M}
& \textbf{56.28} \\
\bottomrule
\end{tabular}%
}
\label{tab:qwen3_8b_gguf}
\vspace{-1.0em}
\end{wraptable}
We finally evaluate whether \methodname{} can improve already-available GGUF K-Quant checkpoints. This setting is practically important because GGUF quantization formats are widely used for local and edge deployment, and many high-quality quantized checkpoints are released publicly. In this setting, \methodname{} treats the GGUF checkpoint as the initialization point, improves its discrete assignments, and projects the result back into the same format. Thus, the optimized model remains compatible with standard GGUF inference stacks. Additional details on the K-Quant parameterization and optimization procedure are provided in Appendix \ref{app:gguf}.

\paragraph{Setup.}
We use publicly available Unsloth GGUF checkpoints as initialization and compare them against the corresponding \methodname{}-optimized checkpoints. We evaluate a reasoning-oriented Qwen3-8B model under two common K-Quant formats: Q3\_K\_M and Q2\_K. The evaluation includes AIME25, GPQA Diamond, and MMLU-Pro. For AIME25 and GPQA Diamond, we report the average \texttt{pass@1} score over 10 and 5 repeated evaluation runs, respectively. In addition to task accuracy, we report the average number of generated tokens, including both visible text tokens and reasoning tokens, since low-bit quantization is known affect not only correctness but also the efficiency of the model's reasoning process; specifically, quantized models tend to produce more tokens for the final answer.

Table~\ref{tab:qwen3_8b_gguf} shows that \methodname{} consistently improves over the Unsloth \citep{unsloth_2025_qwen3_8b} initialization. Gains are largest in the more aggressive Q2\_K format: the average score improves from 50.03 to 56.28 on Qwen3-8B. The Q3\_K\_M setting shows smaller but consistent gains. \methodname{} also often reduces total generated tokens, suggesting that better discrete assignments can improve not only accuracy but also reasoning efficiency. Overall, these results show that \methodname{} extends beyond standard scalar quantization and can improve practical GGUF K-Quant checkpoints without changing their deployment format. We provide additional GGUF K-Quant results on Qwen3-4B in Appendix~\ref{app:qwen3-4b}, where \methodname{} shows similar gains over the corresponding Unsloth checkpoints.

\section{Conclusion}
\label{sec:conclusion}

We presented a sampling-based quantization method suggesting that the apparent divide between ``simple'' scalar quantization and more expressive vector- or trellis-based methods is smaller than previously believed. In the low-bit regime, much of the gap appears to be an \emph{optimization} gap rather than a fundamental limitation of deployment-friendly quantization formats. By turning per-weight grid assignment into a differentiable discrete optimization problem, \methodname{} substantially improves the accuracy of standard symmetric group-wise scalar quantization at 2 and 3 bits, while preserving full compatibility with existing scalar inference kernels and deployment stacks. We further showed that the same idea extends to GGUF K-Quant checkpoints, improving public local-deployment models.

Although the area of weight quantization is by now very well studied, our work presents a new combination of accuracy, simplicity, and kernel compatibility. On dense Llama models, \methodname{} consistently outperforms prior scalar baselines and closes most of the gap to state-of-the-art VQ methods; at ternary precision, it is already competitive with or better than several stronger-bit scalar alternatives. At the same time, because \methodname{} does not rely on learned codebooks or specialized decoding schemes, it remains practical at the scale of modern MoE models, where more expressive quantizers are difficult to train and deploy. 

More broadly, these results indicate that there is still substantial headroom in hardware-friendly scalar quantization, provided that the discrete optimization problem is treated seriously. An important direction for future work is to extend this idea beyond weight-only PTQ, for example to activation and KV-cache quantization, richer blockwise or task-aware objectives, and more efficient relaxations for even lower-bit or jointly quantized settings.

\section*{Acknowledgements}
The authors would like to thank Verda Cloud for computational support, and in particular Paul Chang for his consistent, prompt and generous help throughout the project. We acknowledge the use of Humming kernels developed by Jinzhen Lin and the Venus Team, Ant Group. The ISTA team was supported in part by the FWF Bilateral AI Center of Excellence, as well as a generous grant from the NVIDIA corporation. This work was also supported under project ID 40 as part of the Swiss AI Initiative, through a grant from the ETH Domain and computational resources provided by the Swiss National Supercomputing Centre (CSCS) under the Alps infrastructure. This research was supported by the ISTA Responsible AI Program, made possible through the generous support of Garrett Camp and the Camp Foundation.

\bibliographystyle{plainnat}
\bibliography{neurips_2026} 

\newpage
\appendix

\section{Initialization details}
\label{app:initialization}
In the main text, we state that for the 2-bit case and higher bitwidths we initialize the logits using a Gaussian-like prior around the GPTQ solution. Here we give the details.

For each coordinate $i$ and candidate $k$, we set
\[
(\vl_\text{GPTQ})_{i,k} \propto -\frac{(c_k - \mu_i)^2}{2},
\qquad
\mu_i =
\begin{cases}
(\vq_\text{GPTQ})_i, & b = 2,\\
0, & b > 2.
\end{cases}
\]
Equivalently, candidates closer to $\mu_i$ receive larger initial logits. Here $c_k$ denotes the $k$-th candidate value. For $b=2$, the candidates are the quantization grid points,
\[
c_k \in \mathcal{G}_2,
\]
so centering at $\mu_i=(\vq_\text{GPTQ})_i$ directly favors the GPTQ-selected value. For $b>2$, the candidates are discrete index shifts,
\[
c_k \in \{-2,-1,0,1,2\}.
\]
In this case, the GPTQ solution is already encoded by the initial grid index $j_i^0$, and the candidate $c_k=0$ corresponds to keeping the GPTQ-assigned grid point. Thus centering at $\mu_i=0$ favors staying at the GPTQ solution while still assigning nonzero probability to nearby shifts.

Finally, for each coordinate, we subtract the mean logit and inject noise as in Equation~\ref{eq:logit_init} for the ternary case. This yields an initialization concentrated near the GPTQ solution while still allowing exploration of alternative candidates.

\begin{figure}[t]
\centering
\resizebox{\linewidth}{!}{%
\begin{tikzpicture}[
  font=\footnotesize,
  >=Latex,
]

\definecolor{logitBlue}{RGB}{86,105,190}
\definecolor{gptqRed}{RGB}{200,55,70}
\definecolor{gsqGreen}{RGB}{46,150,92}
\definecolor{winOrange}{RGB}{200,130,60}
\definecolor{bgTint}{RGB}{215,224,245}
\definecolor{arrowGray}{RGB}{110,110,120}
\definecolor{moveArrow}{RGB}{105,105,115}

\tikzset{
  bar/.style         = {draw=logitBlue!85!black, line width=0.28pt, fill=logitBlue!75, rounded corners=0.35pt},
  bar gptq/.style    = {draw=gptqRed!65!black,   line width=0.32pt, fill=gptqRed!88,   rounded corners=0.35pt},
  bar gsq/.style     = {draw=gsqGreen!55!black,  line width=0.32pt, fill=gsqGreen!85,  rounded corners=0.35pt},
  tickline/.style    = {line width=0.42pt, black!75},
  tickfaded/.style   = {line width=0.30pt, black!25},
  gridaxis/.style    = {line width=0.46pt, black!75},
  gptqdot/.style     = {circle, fill=gptqRed,  inner sep=1.0pt, draw=gptqRed!40!black,  line width=0.22pt},
  gsqdot/.style      = {circle, fill=gsqGreen, inner sep=1.0pt, draw=gsqGreen!40!black, line width=0.22pt},
  window/.style      = {draw=winOrange, line width=0.68pt, dashed, rounded corners=2.0pt, dash pattern=on 1.8pt off 1.4pt},
  collabel/.style    = {font=\footnotesize\bfseries},
  rowlabel/.style    = {font=\footnotesize\bfseries, align=center, text=black!90},
  rowsub/.style      = {font=\scriptsize\itshape, text=black!55, align=center},
  trainarr/.style    = {-{Latex[length=2.6mm,width=2.0mm]}, line width=0.90pt, draw=arrowGray},
  movearr/.style     = {-{Latex[length=2.0mm,width=1.5mm]}, line width=0.75pt, draw=moveArrow},
  yaxislab/.style    = {font=\scriptsize\itshape, text=black!60, rotate=90},
  dotlab gptq/.style = {font=\scriptsize\bfseries, text=gptqRed!75!black,  inner sep=1pt},
  dotlab gsq/.style  = {font=\scriptsize\bfseries, text=gsqGreen!60!black, inner sep=1pt},
  leader/.style      = {line width=0.30pt, black!45},
}

\pgfmathsetmacro{\gridW}{4.85}
\pgfmathsetmacro{\gridH}{1.65}
\pgfmathsetmacro{\nticks}{16}
\pgfmathsetmacro{\dx}{\gridW/(\nticks-1)}
\pgfmathsetmacro{\barW}{\dx*0.60}
\pgfmathsetmacro{\rowSep}{3.35}
\pgfmathsetmacro{\colX}{6.55}
\pgfmathsetmacro{\rowLabelX}{-1.55}

\begin{scope}[on background layer]
  \fill[bgTint, fill opacity=0.25, rounded corners=3pt]
    ({\rowLabelX - 0.85}, {-\rowSep - 0.72})
    rectangle
    ({\colX + \gridW + 0.42}, {\gridH + 1.08});
\end{scope}

\node[collabel] at (\gridW/2,         {\gridH + 0.82}) {Before Training};
\node[collabel] at (\colX + \gridW/2, {\gridH + 0.82}) {After Training};

\begin{scope}[yshift=0cm]
  \node[rowlabel] at (\rowLabelX, {\gridH/2 + 0.16}) {Naive};
  \node[rowsub]   at (\rowLabelX, {\gridH/2 - 0.24}) {$2^b$ logits\\per coord.};

  \begin{scope}
    \foreach \i/\h in {
      1/0.020, 2/0.056, 3/0.135, 4/0.278, 5/0.487, 6/0.726, 7/0.923, 8/1.000,
      9/0.923, 10/0.726, 11/0.487, 12/0.278, 13/0.135, 14/0.056, 15/0.020, 16/0.008}
    {
      \pgfmathsetmacro{\x}{(\i-1)*\dx}
      \pgfmathsetmacro{\hh}{\h*\gridH}
      \ifnum\i=8
        \draw[bar gptq] (\x - \barW/2, 0) rectangle (\x + \barW/2, \hh);
      \else
        \draw[bar]      (\x - \barW/2, 0) rectangle (\x + \barW/2, \hh);
      \fi
    }
    \draw[gridaxis] (-0.14, 0) -- (\gridW + 0.14, 0);
    \foreach \i in {1,...,16} {
      \pgfmathsetmacro{\x}{(\i-1)*\dx}
      \draw[tickline] (\x, -0.015) -- (\x, -0.13);
    }
    \pgfmathsetmacro{\xinit}{7*\dx}
    \node[gptqdot] at (\xinit, 0) {};
    \node[dotlab gptq] at (\xinit, -0.36) {GPTQ};
    \node[yaxislab, anchor=south] at (-0.26, {\gridH/2}) {logit};
  \end{scope}

  \begin{scope}[xshift=\colX cm]
    \foreach \i/\h in {
      1/0.005, 2/0.015, 3/0.038, 4/0.085, 5/0.170, 6/0.310, 7/0.525, 8/0.790,
      9/1.000, 10/0.790, 11/0.525, 12/0.310, 13/0.170, 14/0.085, 15/0.038, 16/0.015}
    {
      \pgfmathsetmacro{\x}{(\i-1)*\dx}
      \pgfmathsetmacro{\hh}{\h*\gridH}
      \ifnum\i=9
        \draw[bar gsq] (\x - \barW/2, 0) rectangle (\x + \barW/2, \hh);
      \else
        \draw[bar]     (\x - \barW/2, 0) rectangle (\x + \barW/2, \hh);
      \fi
    }
    \draw[gridaxis] (-0.14, 0) -- (\gridW + 0.14, 0);
    \foreach \i in {1,...,16} {
      \pgfmathsetmacro{\x}{(\i-1)*\dx}
      \draw[tickline] (\x, -0.015) -- (\x, -0.13);
    }
    \pgfmathsetmacro{\xinit}{7*\dx}
    \pgfmathsetmacro{\xsel}{8*\dx}
    \draw[movearr] (\xinit, 0.19)
      .. controls ({\xinit + 0.08}, 0.42) and ({\xsel - 0.08}, 0.42)
      .. (\xsel, 0.19);
    \node[gptqdot] at (\xinit, 0) {};
    \node[gsqdot]  at (\xsel,  0) {};
    \node[dotlab gptq, anchor=east] at ({\xinit - 0.06}, -0.36) {GPTQ};
    \node[dotlab gsq,  anchor=west] at ({\xsel  + 0.06}, -0.36) {GSQ};
    \draw[leader] ({\xinit - 0.05}, -0.28) -- (\xinit, -0.06);
    \draw[leader] ({\xsel  + 0.05}, -0.28) -- (\xsel,  -0.06);
  \end{scope}

  \draw[trainarr] ({\gridW + 0.42}, {\gridH*0.50}) -- ({\colX - 0.42}, {\gridH*0.50});
\end{scope}

\begin{scope}[yshift=-\rowSep cm]
  \node[rowlabel] at (\rowLabelX, {\gridH/2 + 0.16}) {Local Shift};
  \node[rowsub]   at (\rowLabelX, {\gridH/2 - 0.24}) {5 logits\\per coord.};

  \begin{scope}
    \foreach \i/\h in {6/0.25, 7/0.71, 8/1.00, 9/0.71, 10/0.25} {
      \pgfmathsetmacro{\x}{(\i-1)*\dx}
      \pgfmathsetmacro{\hh}{\h*\gridH}
      \ifnum\i=8
        \draw[bar gptq] (\x - \barW/2, 0) rectangle (\x + \barW/2, \hh);
      \else
        \draw[bar]      (\x - \barW/2, 0) rectangle (\x + \barW/2, \hh);
      \fi
    }
    \foreach \i/\lbl in {6/{\text{$-$}2}, 7/{\text{$-$}1}, 8/{0}, 9/{\text{$+$}1}, 10/{\text{$+$}2}} {
      \pgfmathsetmacro{\x}{(\i-1)*\dx}
      \node[font=\tiny, text=winOrange!65!black, inner sep=0.6pt]
        at (\x, {\gridH + 0.12}) {$\lbl$};
    }
    \node[font=\scriptsize\itshape, text=winOrange!65!black, anchor=west]
      at ({9*\dx + 0.42}, {\gridH + 0.12}) {shift $\delta$};

    \draw[gridaxis] (-0.14, 0) -- (\gridW + 0.14, 0);
    \foreach \i in {1,...,16} {
      \pgfmathsetmacro{\x}{(\i-1)*\dx}
      \ifnum\i<6
        \draw[tickfaded] (\x, -0.015) -- (\x, -0.13);
      \else\ifnum\i>10
        \draw[tickfaded] (\x, -0.015) -- (\x, -0.13);
      \else
        \draw[tickline]  (\x, -0.015) -- (\x, -0.13);
      \fi\fi
    }
    \pgfmathsetmacro{\xleft}{5*\dx - \barW/2 - 0.08}
    \pgfmathsetmacro{\xright}{9*\dx + \barW/2 + 0.08}
    \draw[window] (\xleft, -0.20) rectangle (\xright, {\gridH*1.01 + 0.04});

    \pgfmathsetmacro{\xinit}{7*\dx}
    \node[gptqdot] at (\xinit, 0) {};
    \node[dotlab gptq] at (\xinit, -0.40) {GPTQ};
    \node[yaxislab, anchor=south] at (-0.26, {\gridH/2}) {logit};
  \end{scope}

  \begin{scope}[xshift=\colX cm]
    \foreach \i/\h in {6/0.04, 7/0.17, 8/0.48, 9/1.00, 10/0.48} {
      \pgfmathsetmacro{\x}{(\i-1)*\dx}
      \pgfmathsetmacro{\hh}{\h*\gridH}
      \ifnum\i=9
        \draw[bar gsq] (\x - \barW/2, 0) rectangle (\x + \barW/2, \hh);
      \else
        \draw[bar]     (\x - \barW/2, 0) rectangle (\x + \barW/2, \hh);
      \fi
    }
    \foreach \i/\lbl in {6/{\text{$-$}2}, 7/{\text{$-$}1}, 8/{0}, 9/{\text{$+$}1}, 10/{\text{$+$}2}} {
      \pgfmathsetmacro{\x}{(\i-1)*\dx}
      \node[font=\tiny, text=winOrange!65!black, inner sep=0.6pt]
        at (\x, {\gridH + 0.12}) {$\lbl$};
    }
    \node[font=\scriptsize\itshape, text=winOrange!65!black, anchor=west]
      at ({9*\dx + 0.42}, {\gridH + 0.12}) {shift $\delta$};

    \draw[gridaxis] (-0.14, 0) -- (\gridW + 0.14, 0);
    \foreach \i in {1,...,16} {
      \pgfmathsetmacro{\x}{(\i-1)*\dx}
      \ifnum\i<6
        \draw[tickfaded] (\x, -0.015) -- (\x, -0.13);
      \else\ifnum\i>10
        \draw[tickfaded] (\x, -0.015) -- (\x, -0.13);
      \else
        \draw[tickline]  (\x, -0.015) -- (\x, -0.13);
      \fi\fi
    }
    \pgfmathsetmacro{\xleft}{5*\dx - \barW/2 - 0.08}
    \pgfmathsetmacro{\xright}{9*\dx + \barW/2 + 0.08}
    \draw[window] (\xleft, -0.20) rectangle (\xright, {\gridH*1.01 + 0.04});

    \pgfmathsetmacro{\xinit}{7*\dx}
    \pgfmathsetmacro{\xsel}{8*\dx}
    \draw[movearr] (\xinit, 0.19)
      .. controls ({\xinit + 0.08}, 0.42) and ({\xsel - 0.08}, 0.42)
      .. (\xsel, 0.19);
    \node[gptqdot] at (\xinit, 0) {};
    \node[gsqdot]  at (\xsel,  0) {};
    \node[dotlab gptq, anchor=east] at ({\xinit - 0.06}, -0.40) {GPTQ};
    \node[dotlab gsq,  anchor=west] at ({\xsel  + 0.06}, -0.40) {GSQ};
    \draw[leader] ({\xinit - 0.05}, -0.31) -- (\xinit, -0.06);
    \draw[leader] ({\xsel  + 0.05}, -0.31) -- (\xsel,  -0.06);
  \end{scope}

  \draw[trainarr] ({\gridW + 0.42}, {\gridH*0.50}) -- ({\colX - 0.42}, {\gridH*0.50});
\end{scope}

\end{tikzpicture}%
}
\caption{\textbf{Local-shift parameterization at higher bit-widths.}
Each row shows, for a single weight coordinate, the logit distribution over
candidate grid points before and after training. The red bar and dot mark
the GPTQ-initialized grid point used to warm-start the logits; the green
bar and dot mark the grid point selected by \methodname{} after training.
\emph{Top (naive):} placing one trainable logit on every grid point costs
$2^b$ logits per coordinate, which quickly becomes prohibitive as $b$ grows.
\emph{Bottom (local shift):} we instead assign logits only to a discrete
shift $\delta_i \in \{-2,-1,0,+1,+2\}$ relative to the GPTQ-initialized
grid index (dashed window), reducing the per-coordinate parameter count
from $2^b$ to $5$. In both cases the distribution is initialized as a
Gaussian centered at the GPTQ solution.}
\label{fig:local-shift}
\end{figure}%

\section{Vanishing gradients in saturated Gumbel-Softmax relaxations}
\label{app:vanishing_logit_grads}

We explain why gradients through a low-temperature softmax relaxation can become exponentially small when the logits are well separated. Let $a\in\mathbb{R}^K$ denote the logits and let $y=\mathrm{softmax}(a/\tau)$, where $\tau>0$ is the temperature. Suppose that $a$ has a unique maximizer
$i^\star=\arg\max_i a_i$, and define the logit gap
$\Delta \coloneqq a_{i^\star} - \max_{j\neq i^\star} a_j > 0$. As $\tau$
decreases, or as $\Delta$ increases, the relaxed categorical vector $y$ becomes increasingly close to the one-hot vector $e_{i^\star}$. This saturation causes the Jacobian of the softmax to shrink.

\begin{lemma}
\label{lem:vanishing_logit_grads}
Let $a\in\mathbb{R}^K$ have a unique maximizer $i^\star=\arg\max_i a_i$, and define $\Delta \coloneqq a_{i^\star}-\max_{j\neq i^\star} a_j>0$. For $\tau>0$, let $y=\mathrm{softmax}(a/\tau)$. Suppose $\ell$ is differentiable at $y$ and $\|\nabla_y \ell(y)\|_1 \le G$. Then
\[\|\nabla_a \ell(y)\|_\infty \le \frac{G}{\tau}\bigl(1-y_{i^\star}\bigr) \le \frac{(K-1)G}{\tau} e^{-\Delta/\tau}.
\]
Consequently, for fixed $\Delta>0$, $\|\nabla_a \ell(y)\|_\infty \to 0$ as
$\tau\to 0$, and for fixed $\tau>0$, $\|\nabla_a \ell(y)\|_\infty \to 0$ as
$\Delta\to\infty$.
\end{lemma}

\begin{proof}
Let $g=\nabla_y\ell(y)$. The softmax Jacobian satisfies $\partial y_k/\partial a_i = \tau^{-1}y_k(\mathbf{1}_{k=i}-y_i)$, so
\[\frac{\partial \ell}{\partial a_i}=\frac{1}{\tau}\sum_{k=1}^K g_k y_k(\mathbf{1}_{k=i}-y_i).\]
Taking absolute values gives
\[\left|\frac{\partial \ell}{\partial a_i}\right|\le\frac{\|g\|_1}{\tau}\max_k y_k|\mathbf{1}_{k=i}-y_i|.\]
We now bound the final factor. If $i=i^\star$, then
$y_{i^\star}(1-y_{i^\star})\le 1-y_{i^\star}$. If $i\neq i^\star$, then
$y_i(1-y_i)\le y_i\le 1-y_{i^\star}$, and for $k\neq i$,
$y_k y_i \le y_i \le 1-y_{i^\star}$. Thus, for every $i$,
\[\left|\frac{\partial \ell}{\partial a_i}\right|\le\frac{G}{\tau}(1-y_{i^\star}),\]
which implies $\|\nabla_a \ell(y)\|_\infty \le\frac{G}{\tau}(1-y_{i^\star})$. Finally,
\[1-y_{i^\star}=\frac{\sum_{j\neq i^\star}\exp((a_j-a_{i^\star})/\tau)}{1+\sum_{j\neq i^\star}\exp((a_j-a_{i^\star})/\tau)}\le\sum_{j\neq i^\star}\exp((a_j-a_{i^\star})/\tau)\le(K-1)e^{-\Delta/\tau}.\]
The limiting statements follow immediately.
\end{proof}

The lemma shows that the gradient with respect to the logits is controlled by the probability mass assigned to non-maximal categories. Once the relaxation is nearly one-hot, this tail mass is exponentially small in $\Delta/\tau$. Thus, even if the downstream loss has a nonzero gradient with respect to the relaxed variable $y$, the gradient propagated back to the logits can become negligible.

This effect is relevant for optimizer choice. AdamW rescales gradients using an estimate of their second moment and includes an additive numerical constant $\epsilon$ in the denominator. In the saturated regime, the raw gradients passed to the optimizer may be extremely small, so the effective update can become highly sensitive to $\epsilon$ and may shrink enough to stall progress. Increasing $\epsilon$ can partially counteract this behavior, but it introduces an additional tuning choice and may be problem dependent.

Lion instead uses a sign-based update computed from momentum. As a result, its update direction is less directly tied to the absolute magnitude of the instantaneous gradient. This does not remove the saturation phenomenon itself: if the gradient is exactly zero, no first-order optimizer receives useful local information. However, when gradients are nonzero but exponentially small, Lion is less sensitive to their magnitude collapse than AdamW, which makes it a natural choice for training with saturated Gumbel-Softmax relaxations.

\paragraph{Ablation: AdamW vs. Lion} We have theoretically shown that AdamW is not well suited to our setting, as it can effectively stall due to the vanishing gradient magnitudes of the logit parameters, i.e., the parameters associated with discrete assignments. In this subsection, we empirically validate this observation.

Tuning the learning rate for AdamW in this setting is particularly challenging. With a small learning rate, similar to the one used for Lion, the logit parameters initially make progress while the temperature is still sufficiently high. However, after a few epochs, as the temperature decreases and the gradient magnitudes shrink, the logits tend to stall and fail to converge to a clear discrete choice. As a result, the final assignment can become effectively random, since the probabilities of different choices remain close to one another. To mitigate this issue, AdamW requires a much larger learning rate for the logits so that they can converge early in training, before the gradients become too small. This, however, makes the optimization less stable and harder to tune.

Table~\ref{tab:ablation_optimizer} compares AdamW and Lion on Llama-3.1-8B-Instruct under 2-bit weight-only quantization. Lion consistently outperforms AdamW across all evaluated benchmarks, improving the average accuracy from $61.84$ to $65.54$, a gain of $3.70$ points. The largest improvements are observed on WinoGrande and HellaSwag, where Lion improves over AdamW by $4.73$ and $4.55$ points, respectively. These results support our theoretical analysis: Lion provides a more reliable optimization behavior for the logit parameters and leads to better discrete assignment decisions.

\begin{table}[h]
\centering
\caption{
Ablation of the optimizer choice on Llama-3.1-8B-Instruct at 2-bit weight-only quantization.
}
\label{tab:ablation_optimizer}
\begin{tabular}{lcccccc}
\toprule
Optimizer & ARC-C & ARC-E & Hella. & PIQA & Wino. & Avg. \\
\midrule
AdamW & 40.61 & 69.49 & 62.15 & 73.50 & 63.46 & 61.84 \\
Lion & \textbf{44.20} & \textbf{72.18} & \textbf{66.70} & \textbf{76.44} & \textbf{68.19} & \textbf{65.54} \\
\bottomrule
\end{tabular}
\end{table}

\section{Details of within-block staging}
\label{app:within_block_staging}

For Llama and Qwen models, we optimize the quantized layers within each Transformer block using a staged procedure rather than optimizing all layers jointly under a single block reconstruction loss. Joint optimization is the most direct formulation and has been used in prior work~\citep{aqlm,chen2025efficientqat}. However, we found it to be suboptimal in our setting.

One reason is that block reconstruction is only a surrogate for final quantized-model quality. The loss measures the discrepancy between the output of a floating-point block and the output of its quantized counterpart, but this signal is not equally informative for all layers in the block. For layers that appear early in the computation graph, such as the query, key, and value projections, their effect on the block output is mediated by subsequent attention operations, the residual pathway, and the MLP. Consequently, the block-level loss provides a relatively indirect optimization signal for these layers. Later layers, such as the MLP projections, are more directly exposed to the block output loss. This imbalance can make fully joint optimization less effective than optimizing smaller groups of layers under reconstruction losses that are closer to their immediate outputs. This intuition is consistent with the staged strategies used by \citet{quipsharp, qtip}, which also partition the block and optimize one group at a time rather than all layers jointly.

Our staging strategy therefore partitions the block according to the structure of the Transformer computation. Ideally, the query and key projections would be optimized jointly, since the relevant quantity for attention is their interaction through the attention logits rather than the reconstruction of either projection in isolation. In larger models, however, this joint optimization is expensive. As a cheaper approximation, we first optimize the query and key projections independently, each under its own linear reconstruction loss. We then freeze these projections and optimize the value and output projections jointly under a self-attention output reconstruction loss. Finally, after freezing the attention projections, we optimize the MLP projections under the full block reconstruction loss.

\section{Discussion on the effect of calibration data}
\label{sec:exp_kimi26}

The choice of calibration data can substantially affect quantization quality, particularly for Mixture-of-Experts (MoE) models, where sparse routing may leave some experts insufficiently represented during calibration. Moreover, calibration sets concentrated on specific domains may not adequately preserve performance on tasks requiring broader knowledge. We investigate this effect through additional experiments on Kimi-K2.6 and a controlled calibration-data ablation on Qwen3.6-35B-A3B.

\paragraph{Calibration data for Kimi-K2.6.}
The Kimi-K2.5 results in Table~\ref{tab:kimi_combined} show a noticeably larger drop on GPQA Diamond than on the mathematics and coding benchmarks. One plausible explanation is the calibration data: Kimi-K2.5 was calibrated only on OpenThoughts~\citep{guha2025openthoughtsdatarecipesreasoning}, which is concentrated on mathematics and code. This may help preserve performance on MATH500 and LiveCodeBench while providing weaker coverage of science-heavy domains such as GPQA Diamond.

To examine this effect further, we evaluate \methodname{} on Kimi-K2.6 using a more diverse calibration mixture consisting of OpenThoughts, OpenPlatypus~\citep{lee2023platypus}, and FineWeb-Edu~\citep{lozhkov2024fineweb-edu}. The goal is to retain mathematical reasoning coverage while incorporating broader knowledge and science-oriented text. We observed evaluation instabilities on some tasks that we could not conclusively attribute to either the checkpoint or the vLLM execution stack. We therefore report only the tasks for which evaluation was stable: AIME25, GPQA Diamond, and MATH500.

\begin{table}[t]
\centering
\small
\setlength{\tabcolsep}{7pt}
\caption{Additional Kimi-K2.6 results using a diverse calibration mixture.}
\begin{tabular}{lcccc}
\toprule
\textbf{Method} & bit/param & AIME25 & GPQA Diamond & MATH500 \\
\midrule
Base & 4.5 & \textbf{90.00} & \textbf{87.21} & 92.67 \\
\methodname{} (Ours) & 2.13 & 87.78 & 83.16 & \textbf{93.67} \\
\bottomrule
\end{tabular}
\label{tab:kimi26_main}
\end{table}

As shown in Table~\ref{tab:kimi26_main}, the quantized Kimi-K2.6 model remains close to the base model on AIME25 and GPQA Diamond, with drops of 2.22 and 4.05 points, respectively, while slightly improving MATH500 from 92.67 to 93.67. Compared with the Kimi-K2.5 results, the GPQA Diamond degradation is substantially smaller. However, since both the base model and calibration data differ, this comparison does not isolate the contribution of calibration composition.

\paragraph{Calibration data ablation on Qwen3.6-35B-A3B.}
To isolate the effect of calibration data more directly, we perform a controlled experiment on Qwen3.6-35B-A3B at 2-bit precision. We compare two calibration configurations: (i) OpenThoughts alone, and (ii) a mixed dataset consisting of OpenThoughts, OpenPlatypus, and FineWeb-Edu. We evaluate both quantized models on AIME25, GPQA Diamond, and MMLU-Pro, and additionally report the full-precision model as a reference.

\begin{table}[t]
\centering
\small
\setlength{\tabcolsep}{7pt}
\caption{Effect of calibration data on 2-bit \methodname{} quantization of Qwen3.6-35B-A3B. The mixed calibration set combines OpenThoughts, OpenPlatypus, and FineWeb-Edu. Average denotes the mean accuracy across the three benchmarks.}
\begin{tabular}{lcccc}
\toprule
\textbf{Calibration} & AIME25 & GPQA Diamond & MMLU-Pro & Avg. \\
\midrule
Full precision & 93.33 & 83.84 & 84.87 & 87.35 \\
\midrule
OpenThoughts & \textbf{86.67} & 70.71 & 79.76 & 79.05 \\
Mixed & \textbf{86.67} & \textbf{76.77} & \textbf{81.37} & \textbf{81.60} \\
\bottomrule
\end{tabular}
\label{tab:qwen_calibration_ablation}
\end{table}

Table~\ref{tab:qwen_calibration_ablation} shows that calibration composition has a meaningful impact on downstream accuracy. The mixed calibration set improves average performance by 2.55 percentage points over OpenThoughts alone. The largest gain occurs on GPQA Diamond, where accuracy increases from 70.71 to 76.77 (+6.06 points), followed by a 1.61-point improvement on MMLU-Pro. Meanwhile, AIME25 accuracy remains unchanged at 86.67, indicating that the additional calibration diversity improves performance on broader knowledge-intensive tasks without sacrificing mathematical reasoning accuracy in this experiment.

These findings complement the Kimi-K2.6 results and provide direct evidence that calibration-data composition can materially affect low-bit quantization performance. Importantly, calibration selection is orthogonal to the optimization mechanism of \methodname{}: more representative calibration sets, including expert-balanced or self-generated data, can be incorporated without modifying the underlying quantization algorithm. Developing such calibration strategies remains a promising direction for further improving extremely low-bit MoE quantization.

\section{Details of GGUF K-Quant extension}
\label{app:gguf}

The key idea is to treat an existing GGUF checkpoint as the initialization point for discrete optimization. Given a pretrained K-Quant model, each quantized weight is first associated with its current GGUF code or grid index. When the corresponding discrete choice has at most four possible values, as in ternary or 2-bit quantization, we optimize over all possible outcomes directly using the full Gumbel-Softmax relaxation. For larger discrete sets, we instead use the local-shift formulation introduced earlier: rather than assigning logits to all possible code values, we introduce a small discrete shift around the initialized GGUF index. During optimization, Gumbel-Softmax sampling selects either a relaxed code value or a relaxed local shift, depending on the size of the discrete set, while the final model is projected back into the same GGUF K-Quant format. Thus, \methodname{} does not require changing the deployment representation: the output remains a standard K-Quant model and can be served using the same GGUF-compatible inference stack. In this sense, \methodname{} acts as a post-processing optimizer for existing GGUF checkpoints, improving the discrete assignments while preserving the original storage and execution format.

\section{Limitations}

\label{app:limitations}

While \methodname{} improves the accuracy of low-bit scalar quantization and preserves deployment-friendly formats, it has several limitations. First, the method introduces auxiliary logits during optimization, with memory overhead proportional to the number of candidate grid values considered per weight. Although this overhead is temporary and does not affect the final deployed model, it makes full-model end-to-end optimization impractical for large LLMs and motivates our block-wise, expert-wise, and local-shift training procedures.

Second, \methodname{} depends on a good initialization. In our experiments we warm-start from GPTQ or from an existing GGUF checkpoint, and the local-shift formulation assumes that most useful assignment changes remain near the initialized grid point. This assumption works well empirically, but it may restrict recovery when the initialization is poor or when globally different assignments are needed. More adaptive shift ranges or multi-stage re-initialization strategies may improve robustness.

Finally, our results do not imply that scalar quantization universally matches vector- or trellis-based methods. These methods remain more expressive at a fixed bit-width and may retain an advantage in regimes where codebook lookup or specialized kernels are acceptable. Rather, our results show that a carefully optimized scalar quantizer can close much of the gap while preserving a simpler and more portable deployment format.

\section{Additional experimental details}
\label{app:extra_exp}

\subsection{Full training hyperparameters}
\label{app:train_hparams}

Tables~\ref{tab:hparams_llama}, \ref{tab:hparams_kimi}, and \ref{tab:hparams_qwen} summarize the training hyperparameters used in our block-wise optimization experiments.

For the dense Llama and Qwen models, we use 20 epochs of block-wise optimization. For Kimi models, we use only 10 epochs. Although Kimi models are much larger overall, the optimization problem is decomposed across 384 experts and solved independently for each expert. As a result, each individual optimization problem is substantially smaller than in the dense Llama models, which makes fewer epochs sufficient in practice. Additionally, based on the discussion in the previous subsection, for Qwen models we use a calibration mixture consisting of OpenThoughts, OpenPlatypus, and FineWeb-Edu.

There is one additional stabilization detail for Llama-3.1-70B. For this model only, we apply gradient clipping during training, and only to the logits of the discrete parameters, not to the group scales. The clipping threshold is set to $10^{-6}$ for the 2-bit setting and $10^{-8}$ for the 3-bit setting.

\begin{table*}[t]
\centering
\small
\setlength{\tabcolsep}{4pt}
\caption{Training hyperparameters for the Llama experiments.}
\resizebox{\textwidth}{!}{
\begin{tabular}{c|c|c|c|c|c|c|c|c|c|c|c|c|c}
\toprule
Bit-width & Logits lr & Group scales lr & Weight decay & Betas & Epochs &\# Seqs. & Seq. len. & Batch size & Group size & $\tau$ schedule & $\kappa$ schedule & $\alpha$ & std\\
\midrule
$1.58$-bit
& 1e-4
& 5e-5
& $1.0$
& $(0.9, 0.95)$
& 20
& 4096
& 4096
& 64
& 128
& linear: $2 \rightarrow 0.05$
& linear: $100 \rightarrow 500$
& 3
& $0.01$
\\
\midrule
$2/3$-bit
& 1e-4
& 5e-5
& $1.0$
& $(0.9, 0.95)$
& 20
& 4096
& 4096
& 64
& 128
& linear: $2 \rightarrow 0.05$
& linear: $100 \rightarrow 500$
& 6
& $0.01$
\\
\bottomrule
\end{tabular}
}
\label{tab:hparams_llama}
\end{table*}

\begin{table*}[t]
\centering
\small
\setlength{\tabcolsep}{4pt}
\caption{Training hyperparameters for Kimi experiments.}
\resizebox{\textwidth}{!}{
\begin{tabular}{c|c|c|c|c|c|c|c|c|c|c|c|c|c}
\toprule
Bit-width & Logits lr & Group scales lr & Weight decay & Betas & Epochs & \# Seqs. & Seq. len. & Batch size & Group size & $\tau$ schedule & $\kappa$ schedule & $\alpha$ & std \\
\midrule
$2$-bit
& 2e-4
& 1e-5
& $1.0$
& $(0.9, 0.95)$
& 10
& 4096
& 4096
& 64
& 128
& linear: $2 \rightarrow 0.05$
& linear: $100 \rightarrow 500$
& 6
& $0.01$
\\
\bottomrule
\end{tabular}
}
\label{tab:hparams_kimi}
\end{table*}

\begin{table*}[t]
\centering
\small
\setlength{\tabcolsep}{4pt}
\caption{Training hyperparameters for the Qwen experiments.}
\resizebox{\textwidth}{!}{
\begin{tabular}{c|c|c|c|c|c|c|c|c|c|c|c|c|c}
\toprule
Bit-width & Logits lr & Group scales lr & Weight decay & Betas & Epochs &\# Seqs. & Seq. len. & Batch size & Group size & $\tau$ schedule & $\kappa$ schedule & $\alpha$ & std\\
\midrule
Q2\_K
& 1e-4
& 2e-5
& $1.0$
& $(0.9, 0.95)$
& 20
& 4096
& 4096
& 64
& 128
& linear: $2 \rightarrow 0.05$
& linear: $100 \rightarrow 500$
& 6
& $0.01$
\\
\midrule
Q3/4/5\_K
& 1e-4
& 0
& $1.0$
& $(0.9, 0.95)$
& 20
& 4096
& 4096
& 64
& 128
& linear: $2 \rightarrow 0.05$
& linear: $100 \rightarrow 500$
& 6
& $0.01$
\\
\bottomrule
\end{tabular}
}
\label{tab:hparams_qwen}
\end{table*}

\subsection{Evaluation}
We evaluate all models with \texttt{lm-eval-harness}~\citep{eval-harness} in the zero-shot setting, using a maximum sequence length of 4096. For the Llama models, we report accuracy on ARC-Easy, ARC-Challenge, HellaSwag, WinoGrande, and PIQA~\citep{clark2018think, zellers2019hellaswag, sakaguchi2021winogrande, bisk2020piqa}. These are standard zero-shot reasoning and commonsense benchmarks widely used in the quantization literature, and together they cover multi-choice scientific reasoning, commonsense completion, and physical reasoning.

For Kimi-K2.5 model, we additionally focus on long-context and reasoning evaluations. More specifically, for long-context, we evaluate models on OpenAI-MRCR \citep{vodrahalli2024michelangelo} benchmark across all sequence-length buckets from 0 to 256k, which is the model's maximum supported sequence length. For each sequence length bucket, we report average \texttt{pass@1} score over 5 repetitions. As for reasoning evaluations, we focus on AIME25 \citep{aime25}, GPQA:Diamond \citep{rein2024gpqa}, MATH500 \citep{lightman2023let}, and LiveCodeBench-v6 \citep{jain2024livecodebench}. We report the average \texttt{pass@1} score: over 10 repetitions for AIME25 and LiveCodeBench-v6, and over 5 repetitions for GPQA:Diamond and MATH500. For both, long-context and reasoning evaluations, we follow Kimi's suggested sampling parameters: \texttt{temperature=1.0} and \texttt{top\_p=0.95}.

\subsection{Randomness and reproducibility}
\methodname{} is stochastic due to Gumbel-Softmax sampling. However, repeated runs on early-layer block-wise reconstruction show low variance in the optimization loss, hence given the cost of 70B- and MoE-scale quantization, we report a single run per configuration. All experiments were conducted primarily on a node with $8\times$H200.

\subsection{End-to-end compression runtime}

Table~\ref{tab:runtimes} summarizes the end-to-end runtime of \methodname{} across different models using $8\times$H200 GPUs. As expected, the total runtime increases with model size. The 8B model completes in 10 hours, while Llama-3.1-70B-Instruct requires 68 hours. Interestingly, although Kimi-K2.5 is substantially larger than Llama-3.1-70B-Instruct, it completes in only 24 hours. This is mainly due to two factors. First, we do not quantize the attention layers for Kimi-K2.5, which decreases its computational cost. Second, we train Kimi-K2.5 for fewer epochs, which further reduces its total runtime. Overall, these results show that \methodname{} remains practical at different model scales, although larger models introduce a substantially higher runtime overhead.

\begin{table*}[t]
\centering
\begin{tabular}{lc}
\toprule
Model & Runtime \\
\midrule
Llama-3.1-8B-Instruct & 10 hours \\
Llama-3.1-70B-Instruct & 68 hours \\
Kimi-K2.5 & 24 hours \\
\bottomrule
\end{tabular}
\caption{Total runtime of \methodname{} for different models on $8\times$H200 GPUs.}
\label{tab:runtimes}
\end{table*}

\section{Effect of end-to-end scale-only fine-tuning}
\label{app:scale_ft}

Table~\ref{tab:scale_ft} isolates the contribution of the end-to-end scale-only fine-tuning stage for the 2-bit Llama models. In this stage, the discrete assignments found by block-wise optimization are kept fixed, and only the per-group scales are updated using the distillation objective described in Section~\ref{sec:exp_llama}. This ablation shows how much additional performance can be recovered by a lightweight global refinement after the block-wise discrete search.

\begin{table*}[t]
\centering
\setlength{\tabcolsep}{3.5pt}
\caption{Effect of end-to-end scale-only fine-tuning on the 2-bit GSQ models.}
\resizebox{\textwidth}{!}{%
\begin{tabular}{l|cccccc|cccccc}
\toprule
& \multicolumn{6}{c|}{\textbf{Llama-3.1-8B-Instruct}} & \multicolumn{6}{c}{\textbf{Llama-3.1-70B-Instruct}} \\
\textbf{Setting}
& ARC-C & ARC-E & Hella. & PIQA & Wino. & Avg.
& ARC-C & ARC-E & Hella. & PIQA & Wino. & Avg. \\
\midrule
GSQ, block-wise only
& 44.20 & 72.18 & 66.70 & 76.44 & 68.19 & 65.54
& 57.25 & 79.71 & 78.08 & 80.09 & 77.11 & 74.45 \\
GSQ, + scale fine-tuning
& 48.12 & 72.35 & 73.42 & 78.07 & 70.80 & 68.55
& 58.87 & 79.55 & 82.11 & 81.07 & 76.24 & 75.57 \\
$\Delta$
& +3.92 & +0.17 & +6.72 & +1.63 & +2.61 & +3.01
& +1.62 & -0.16 & +4.03 & +0.98 & -0.87 & +1.12 \\
\bottomrule
\end{tabular}%
}
\label{tab:scale_ft}
\end{table*}

The results in Table~\ref{tab:scale_ft} show that the block-wise stage already captures the main benefit of the discrete optimization, while a single end-to-end pass that updates only the scales can provide an additional refinement without revisiting the discrete assignments.

\section{Connection to MaskLLM: a 2:4 sparsity comparison}
\label{app:maskllm_compare}

Our method is directly inspired by MaskLLM~\citep{maskllm}. Conceptually, \methodname{} extends the same discrete optimization viewpoint from structured sparsity to low-bit quantization, while also changing the optimization granularity from end-to-end training to block-wise training. Since the latter is substantially cheaper, it is natural to ask whether the block-wise formulation remains competitive even in the original setting for which MaskLLM was designed.

To answer this, we perform an additional experiment on Llama-2-7B in the original 2:4 structured sparsity setting of MaskLLM. This experiment isolates the optimization strategy from the representation format: rather than comparing sparsity to quantization, we compare end-to-end MaskLLM training to our block-wise formulation on the same sparsity task. Importantly, the two methods do not optimize exactly the same variables. MaskLLM learns only the binary sparsity mask while keeping the underlying dense weights fixed at their original pretrained values. In contrast, in our formulation we jointly optimize both the mask and the weight values. Therefore, this comparison should not be interpreted as a strictly matched ablation, but rather as evidence that the proposed optimization framework remains effective, and in practice stronger, even when applied back to the structured sparsity setting that originally motivated MaskLLM. As shown in Table~\ref{tab:maskllm_compare}, the block-wise variant yields a better average zero-shot score across the same five tasks used in the main text.

\begin{table}[t]
\centering
\caption{Comparison to MaskLLM on the original 2:4 structured sparsity task of~\citet{maskllm}, evaluated on Llama-2-7B.}
\scalebox{1.0}{
\begin{tabular}{l|cccccc}
\toprule
\textbf{Method} & ARC-C & ARC-E & Hella. & PIQA & Wino. & Avg. \\
\midrule
FP16/BF16  & 46.25 & 74.58 & 75.98 & 79.11 & 69.14 & 69.01 \\
\midrule
MaskLLM (end-to-end, mask-only) & 37.97 & 64.98 & 68.27 & 76.22 & 65.19 & 62.53  \\
\methodname{} (block-wise, mask + weights)  & 40.02 & 68.18 & 65.31 & 75.35 & 65.43 & 62.86 \\
\bottomrule
\end{tabular}}
\label{tab:maskllm_compare}
\end{table}

\section{Qwen3-4B GGUF results}
\label{app:qwen3-4b}

Table~\ref{tab:qwen3_4b_gguf} reports GGUF K-Quant results for Qwen3-4B. 
Although these results are included mainly for completeness, they show a consistent advantage for \methodname{} over the Unsloth baselines at both Q3\_K\_M and Q2\_K quantization levels. The improvement is especially pronounced under the more aggressive Q2\_K setting, where \methodname{} substantially recovers accuracy across all evaluated benchmarks while using comparable or fewer tokens on GPQA Diamond and MMLU-Pro.

\begin{table*}[t]
\centering
\setlength{\tabcolsep}{5pt}
\caption{GGUF K-Quant results on Qwen3-4B. Tokens include text and reasoning tokens.}
\small
\begin{tabular}{l|cc|cc|cc|c}
\toprule
\multirow{2}{*}{Method}
& \multicolumn{2}{c|}{AIME25}
& \multicolumn{2}{c|}{GPQA Diamond}
& \multicolumn{2}{c|}{MMLU-Pro}
& \multirow{2}{*}{Avg.} \\
& Acc. & Tokens 
& Acc. & Tokens 
& Acc. & Tokens 
& \\
\midrule
BF16
& 62.00 & 0.48M
& 53.74 & 1.29M
& 67.69 & 30.02M
& 61.14 \\
\midrule
Unsloth-Q3\_K\_M
& 55.00 & 0.60M
& 47.98 & \textbf{1.32M}
& 64.11 & 34.00M
& 55.70 \\
\methodname{}-Q3\_K\_M
& \textbf{55.67} & \textbf{0.51M}
& \textbf{52.53} & \textbf{1.32M}
& \textbf{65.28} & \textbf{31.69M}
& \textbf{57.83} \\
\midrule
Unsloth-Q2\_K
& 23.00 & \textbf{0.51M}
& 28.59 & 2.17M
& 53.43 & 45.97M
& 35.01 \\
\methodname{}-Q2\_K
& \textbf{45.33} & 0.60M
& \textbf{40.10} & \textbf{1.93M}
& \textbf{59.60} & \textbf{44.69M}
& \textbf{48.34} \\
\bottomrule
\end{tabular}%
\label{tab:qwen3_4b_gguf}
\end{table*}

\section{Kimi-K2.5 additional results}
\label{sec:exp_agentic}

\begin{table}[t]
\centering
\caption{Results for Kimi-K2.5 for average performance ($\uparrow$) across agentic benchmarks, spanning: Operating System (OS) Tasks, Database (DB) Task, Knowledge Graph (KG) Tasks, Web Shopping (WS) Tasks and ALFWorld, an Text Adventure Game (Text Game / House-Holding, HH). Reported values show mean score of 3 runs.}
\begin{tabular}{l|c|c|c|c|c}
\toprule
Method & OS & DB & KG & AlfWorld & WS \\
\midrule
Base & 50 & 69 & 57 & 53 & 52 \\ \midrule
Unsloth\_UD-IQ1\_M (2.4bit) & 40 & 61 & 39 & 41 & 47 \\
GSQ (2.37bit) & \textbf{43} & \textbf{63} & \textbf{45} & \textbf{42} & \textbf{49} \\
\bottomrule
\end{tabular}
\label{tab:agentic_benchmark}
\end{table}

Since Kimi-K2.5 is a “native multimodal agentic model” according to their model card \cite{team2026kimi}, we wanted to check whether our quantization method preserves its ability to perform agentic tasks. For this, we evaluate Kimi-K2.5, our 2-bit version, and a similarly sized and bpw-comparable Unsloth model on AgentBench \cite{liu2025agentbenchevaluatingllmsagents}. We used the most current version as introduced in \cite{zhang2025agentrlscalingagenticreinforcement}, which consists of 5 agentic tasks spanning Operating System (OS) tasks, Database (DB) tasks, Knowledge Graph (KG) tasks, Web Shopping (WS) tasks, and ALFWorld, a text adventure game (Text Game / House-Holding, HH). We evaluate all models with 3 runs on each benchmark and report mean. As shown in Table \ref{tab:agentic_benchmark}, we remain competitive compared to the baseline on database, operating system, and web shopping tasks, which are among the most widely adopted settings. For tasks involving knowledge graphs and textual games, we perform worse, although this could heavily depend on the calibration data.
Additionally, we outperform Unsloth's model on every task while maintaining a comparable overall model size. However, we note that Unsloth's model uses non-uniform quantization and quantizes every layer, whereas our approach only quantizes the experts. Therefore, although the total model size is roughly equivalent, our method effectively operates at a lower average bitwidth within the expert layers.

\paragraph{Limitations.} Furthermore, we want to highlight that many issues occurred while evaluating GGUF models, making it infeasible to run more challenging evaluations with them. Templating often failed, and evaluating Kimi-K2.5 became extremely slow and impractical for evaluation purposes. Therefore, we focused on the most practical evaluations and skipped the others.

\section{Ablations}
\subsection{Naive block optimization}
\label{app:ablation_block_optimization}

In the main experiments, we do not optimize all quantized tensors in a Transformer block jointly under a single block reconstruction loss. Instead, we use the staged within-block schedule described in Section~4.1.

To validate this choice, we compare against two simpler alternatives on Llama-3.1-8B-Instruct at 2-bit weight-only quantization. The first baseline, \emph{Only-linear}, skips block-level optimization entirely and optimizes each linear layer only with its own layerwise reconstruction loss. The second baseline, \emph{Block-all}, performs the most direct form of block optimization: all quantized tensors inside a Transformer block are optimized jointly using the block output reconstruction loss. Since these ablations are designed to isolate the effect of the block-wise optimization schedule, all numbers in this table are reported \emph{before} the final end-to-end scale-only fine-tuning stage. For the staged GSQ result, we therefore use the ``GSQ, block-wise only'' setting from Table~\ref{tab:scale_ft}, rather than the final scale-tuned result used in the main comparison.

\begin{table}[t]
\centering
\caption{
Ablation of the optimization objective and within-block schedule on Llama-3.1-8B-Instruct at 2-bit weight-only quantization.
}
\label{tab:ablation_block_optimization}
\begin{tabular}{lcccccc}
\toprule
Method & ARC-C & ARC-E & Hella. & PIQA & Wino. & Avg. \\
\midrule
Only-linear & 37.46 & 63.97 & 56.54 & 70.13 & 60.30 & 57.68 \\
Block-all & 41.64 & 66.88 & 64.69 & 72.47 & 64.40 & 62.02 \\
\methodname{}, block-wise only & \textbf{44.20} & \textbf{72.18} & \textbf{66.70} & \textbf{76.44} & \textbf{68.19} & \textbf{65.54} \\
\bottomrule
\end{tabular}
\end{table}

The results show that block-level information is useful, but that using it naively is insufficient. Moving from \emph{Only-linear} to \emph{Block-all} improves the average score from 57.68 to 62.02, confirming that reconstruction objectives beyond independent linear layers help reduce accumulated quantization error. However, \emph{Block-all} is still 3.52 points below the staged \methodname{} schedule before any end-to-end scale-only fine-tuning. The gap is consistent across all tasks: staged \methodname{} improves over \emph{Block-all} by 2.56 points on ARC-Challenge, 5.30 on ARC-Easy, 2.01 on HellaSwag, 3.97 on PIQA, and 3.79 on WinoGrande.

\subsection{Robustness to initialization}
\label{app:init_ablation}

We investigate the sensitivity of \methodname{} to its initialization through two controlled experiments on Llama-3.1-8B-Instruct at 2-bit precision. All results in this subsection are obtained after block-wise optimization, before the end-to-end scale-only fine-tuning stage.

\paragraph{Initial quantization assignments.}
We first examine the effect of the initial quantization assignments by replacing GPTQ with round-to-nearest (RTN) and random Gaussian initialization, while retaining the same Gaussian-like logit prior centered on the selected assignments. Table~\ref{tab:init_assignment_ablation} reports the downstream zero-shot accuracy for each initialization.

\begin{table}[t]
\centering
\small
\setlength{\tabcolsep}{5pt}
\caption{Effect of initial quantization assignments on 2-bit \methodname{} for Llama-3.1-8B-Instruct, before end-to-end fine-tuning.}
\label{tab:init_assignment_ablation}
\begin{tabular}{lcccccc}
\toprule
Initialization & ARC-C & ARC-E & Hella. & PIQA & Wino. & Avg. \\
\midrule
Random & 22.95 & 28.91 & 26.37 & 52.01 & 51.22 & 36.29 \\
RTN & \textbf{44.71} & \textbf{72.18} & 65.73 & 75.30 & \textbf{68.67} & 65.32 \\
GPTQ & 44.20 & \textbf{72.18} & \textbf{66.70} & \textbf{76.44} & 68.19 & \textbf{65.54} \\
\bottomrule
\end{tabular}
\end{table}

Random initialization substantially degrades performance, indicating that a meaningful initial assignment is important for effective optimization. In contrast, RTN and GPTQ achieve comparable downstream accuracy, with GPTQ providing a slightly higher average score (65.54 versus 65.32). These results suggest that \methodname{} is relatively robust to the choice of a structured initial assignment, rather than relying specifically on GPTQ.

\paragraph{Logit initialization.}
We next isolate the effect of the geometry-aware Gaussian-like logit prior described in Appendix~\ref{app:initialization}. Keeping the GPTQ assignments fixed, we compare our initialization against a simpler alternative that increases only the logit corresponding to the GPTQ-selected grid point, leaving all other logits unchanged. This alternative does not explicitly encode distances between candidate grid points.

Table~\ref{tab:init_prior_loss} reports the hard validation reconstruction losses during optimization of the first transformer block.

\begin{table}[t]
\centering
\small
\setlength{\tabcolsep}{6pt}
\caption{Hard validation reconstruction losses for different logit initialization strategies on the first block of Llama-3.1-8B-Instruct at 2-bit precision.}
\label{tab:init_prior_loss}
\begin{tabular}{llccccc}
\toprule
Initialization & Target & E2 & E4 & E6 & E8 & E10 \\
\midrule
Gaussian-like & Attention & 7.68e-6 & 2.67e-6 & 2.17e-6 & 1.99e-6 & \textbf{1.94e-6} \\
Simple shift & Attention & 4.36e-6 & 4.31e-6 & 4.28e-6 & 4.27e-6 & 4.26e-6 \\
\midrule
Gaussian-like & Full block & 1.42e-5 & 1.11e-5 & 9.92e-6 & 9.42e-6 & \textbf{9.27e-6} \\
Simple shift & Full block & 2.41e-5 & 2.36e-5 & 2.35e-5 & 2.35e-5 & 2.34e-5 \\
\bottomrule
\end{tabular}
\end{table}

The simple logit initialization quickly stagnates, whereas the Gaussian-like prior continues improving throughout optimization. By epoch 10, the geometry-aware initialization reduces the self-attention and full-block reconstruction losses by 54.5\% and 60.4\%, respectively, relative to the simple shift.

This improvement also translates into substantially higher downstream accuracy, as shown in Table~\ref{tab:init_prior_accuracy}.

\begin{table}[t]
\centering
\small
\setlength{\tabcolsep}{5pt}
\caption{Downstream zero-shot accuracy of different logit initialization strategies for 2-bit Llama-3.1-8B-Instruct, before end-to-end fine-tuning.}
\label{tab:init_prior_accuracy}
\begin{tabular}{lccccc}
\toprule
Initialization & ARC-C & ARC-E & Hella. & PIQA & Wino. \\
\midrule
Simple shift & 26.02 & 45.37 & 46.79 & 62.68 & 53.35 \\
Gaussian-like & \textbf{44.20} & \textbf{72.18} & \textbf{66.70} & \textbf{76.44} & \textbf{68.19} \\
\bottomrule
\end{tabular}
\end{table}

Overall, these experiments distinguish two aspects of initialization. While \methodname{} is relatively insensitive to the choice between reasonable initial assignments such as RTN and GPTQ, the representation of these assignments in the relaxed optimization is considerably more important. Encoding the geometry of the scalar quantization grid through the Gaussian-like logit prior leads to substantially better reconstruction and downstream performance.

\subsection{Validating the local-shift parameterization}
\label{app:ablation_local_shift}

For bit-widths above 2, \methodname{} uses a local-shift parameterization rather than optimizing over all grid values. In particular, for 3-bit quantization the full Gumbel-Softmax relaxation would assign one logit to each of the 8 scalar grid values, whereas our local-shift formulation only considers shifts in $\{-2,-1,0,1,2\}$ around the initialized GPTQ assignment. This reduces the number of logits per weight from 8 to 5 while preserving the most relevant local moves.

To test whether this restriction excludes important assignments, we run an ablation on Llama-3.1-8B-Instruct in which the 3-bit quantizer is optimized using the \emph{full} 8-way Gumbel-Softmax relaxation. After training, we compare the final hard grid index of each weight to its initialized GPTQ grid index and count the absolute index displacement. Table~\ref{tab:local_shift_ablation} reports the resulting distribution across several projection types.
\begin{table}[t]
\centering
\small
\setlength{\tabcolsep}{6pt}
\caption{
Ablation validating the local-shift parameterization.
}
\label{tab:local_shift_ablation}
\begin{tabular}{lcccc}
\toprule
Tensor & $\Delta=0$ & $\Delta=1$ & $\Delta=2$ & $\Delta=3$ \\
\midrule
\texttt{q\_proj}    & 86.261\% & 13.662\% & 0.076\% & 0.00004\% \\
\texttt{k\_proj}    & 84.933\% & 14.957\% & 0.110\% & 0.00005\% \\
\texttt{v\_proj}    & 78.711\% & 21.289\% & 0.000\% & 0.00000\% \\
\texttt{o\_proj}    & 85.660\% & 14.340\% & 0.000\% & 0.00000\% \\
\texttt{gate\_proj} & 82.671\% & 17.329\% & 0.000\% & 0.00000\% \\
\texttt{up\_proj}   & 83.925\% & 16.075\% & 0.000\% & 0.00000\% \\
\texttt{down\_proj} & 86.954\% & 13.046\% & 0.000\% & 0.00000\% \\
\midrule
Total      & 84.635\% & 15.357\% & 0.008\% & 0.000004\% \\
\bottomrule
\end{tabular}
\end{table}
The results strongly support the local-shift design. Across the optimized weights, 84.64\% remain exactly at their initialized grid index and 15.36\% move by only one grid position. Only 0.008\% move by two positions, and 0.000004\% move by three positions. Equivalently, 99.999996\% of all observed full-relaxation assignments lie within the $\{-2,-1,0,1,2\}$ neighborhood used by our local-shift formulation.

\subsection{Gumbel--Softmax versus deterministic softmax and straight-through estimators}
\label{app:softmax_vs_gumbel}

The key role of Gumbel--Softmax in \methodname{} is to learn a distribution over candidate grid assignments while allowing gradients to propagate through relaxed assignments. Unlike a straight-through estimator (STE), the forward pass uses soft values at finite temperatures rather than always selecting a discrete candidate. In the absence of perturbations, deterministic softmax may increasingly concentrate around early choices, limiting subsequent exploration of alternative assignments.

To examine the importance of stochasticity and relaxation separately, we repeat the paper's 2-bit experiment on the first transformer block of Llama-3.1-8B-Instruct, comparing (i) the proposed Gumbel--Softmax relaxation, (ii) deterministic softmax without noise, (iii) hard argmax with STE, and (iv) Gumbel perturbations followed by hard argmax with STE. All other experimental settings are unchanged. Table~\ref{tab:gsq_estimator_ablation} presents the hard validation losses over ten epochs.

\begin{table*}[t]
\centering
\scriptsize
\setlength{\tabcolsep}{3pt}
\caption{Ablation of categorical assignment mechanisms on the first transformer block of Llama-3.1-8B-Instruct (2-bit). The only changed component is the assignment relaxation/estimator. Entries are hard validation reconstruction losses (lower is better).}
\label{tab:gsq_estimator_ablation}
\resizebox{\textwidth}{!}{%
\begin{tabular}{llcccccccccc}
\toprule
\textbf{Method / noise} & \textbf{Target} & \textbf{E1} & \textbf{E2} & \textbf{E3} & \textbf{E4} & \textbf{E5} & \textbf{E6} & \textbf{E7} & \textbf{E8} & \textbf{E9} & \textbf{E10} \\
\midrule
\textbf{Gumbel-Softmax} & Self-attention & 7.76e-6 & 7.68e-6 & 4.13e-6 & 2.67e-6 & 2.35e-6 & 2.17e-6 & 2.06e-6 & 1.99e-6 & 1.96e-6 & 1.94e-6 \\
\textbf{Gumbel-Softmax} & Full block & 1.87e-5 & 1.42e-5 & 1.22e-5 & 1.11e-5 & 1.04e-5 & 9.92e-6 & 9.61e-6 & 9.42e-6 & 9.33e-6 & 9.27e-6 \\
Deterministic softmax & Self-attention & 7.17e-6 & 5.83e-6 & 5.14e-6 & 4.95e-6 & 4.62e-6 & 4.45e-6 & 4.40e-6 & 4.40e-6 & 4.14e-6 & 3.51e-6 \\
Deterministic softmax & Full block & 2.18e-5 & 1.82e-5 & 1.81e-5 & 1.83e-5 & 1.84e-5 & 1.83e-5 & 1.81e-5 & 1.78e-5 & 1.74e-5 & 1.52e-5 \\
Argmax + STE & Self-attention & 5.35e-6 & 4.56e-6 & 4.51e-6 & 4.37e-6 & 4.19e-6 & 4.08e-6 & 3.93e-6 & 3.66e-6 & 3.53e-6 & 3.40e-6 \\
Argmax + STE & Full block & 1.88e-5 & 1.78e-5 & 1.63e-5 & 1.65e-5 & 1.61e-5 & 1.55e-5 & 1.55e-5 & 1.51e-5 & 1.49e-5 & 1.42e-5 \\
Gumbel + STE & Self-attention & 8.24e-6 & 5.25e-6 & 3.90e-6 & 3.39e-6 & 3.01e-6 & 2.73e-6 & 2.55e-6 & 2.44e-6 & 2.37e-6 & 2.33e-6 \\
Gumbel + STE & Full block & 1.88e-5 & 1.44e-5 & 1.32e-5 & 1.24e-5 & 1.18e-5 & 1.13e-5 & 1.09e-5 & 1.06e-5 & 1.04e-5 & 1.03e-5 \\
\bottomrule
\end{tabular}}
\end{table*}

Gumbel--Softmax achieves the lowest final hard reconstruction error on both targets: $1.94\times10^{-6}$ for self-attention and $9.27\times10^{-6}$ for the full block. Gumbel+STE is the next-best alternative, reaching $2.33\times10^{-6}$ and $1.03\times10^{-5}$, respectively. Both stochastic approaches improve steadily toward the end of training, whereas the deterministic alternatives make smaller late-stage gains, particularly on the full block. At low temperatures, nearly discrete assignments may leave scale refinement as the main source of further improvement. These results indicate that stochastic exploration of the categorical assignment space is beneficial, and that the differentiable relaxation provides additional gains over hard Gumbel+STE under the tested optimization schedule.

\subsection{Noise ablation}
\label{app:noise_ablation}

\paragraph{Motivation.}
\methodname{} optimizes a categorical assignment over scalar quantization grid points. For assignment logits $\ell_1,\ldots,\ell_K$, the Gumbel--Max trick yields
\begin{equation}
\Pr\!\left[\argmax_{k}\{\ell_k+g_k\}=i\right]
=\frac{\exp(\ell_i)}{\sum_{k=1}^K\exp(\ell_k)},
\qquad g_k\overset{\mathrm{i.i.d.}}{\sim}\mathrm{Gumbel}(0,1).
\label{eq:gsq_gumbel_max}
\end{equation}
Thus, additive Gumbel perturbations have a principled connection to the categorical distribution defined by the logits. The temperature-controlled softmax relaxation of the perturbed logits retains differentiability and approaches a categorical sample as the temperature tends to zero~\citep{maddison2016concrete,jang2016categorical}. Gumbel noise is therefore a canonical choice for sampling from learned categorical probabilities; this observation does not establish that Gumbel noise is empirically superior to every other perturbation distribution.

\paragraph{Choice of perturbation distribution.}
To isolate the role of the noise family, we repeat the 2-bit experiment on the first transformer block of Llama-3.1-8B-Instruct and replace the injected Gumbel noise with Gaussian, logistic, or uniform noise. We use four quantization candidates per coordinate. The perturbations are mean-centered and matched to the Gumbel(0,1) variance of $\pi^2/6$; initialization, optimizer, calibration data, and annealing schedules are fixed, without tuning separately for each family. Table~\ref{tab:gsq_noise_family} reports hard validation reconstruction losses throughout training.

\begin{table*}[t]
\centering
\scriptsize
\setlength{\tabcolsep}{3pt}
\caption{Noise-family ablation on the first transformer block of Llama-3.1-8B-Instruct at 2-bit precision. All perturbations are mean-centered and variance-matched to Gumbel(0,1), with otherwise identical training settings. Entries are hard validation reconstruction losses (lower is better).}
\label{tab:gsq_noise_family}
\resizebox{\textwidth}{!}{%
\begin{tabular}{llcccccccccc}
\toprule
\textbf{Method / noise} & \textbf{Target} & \textbf{E1} & \textbf{E2} & \textbf{E3} & \textbf{E4} & \textbf{E5} & \textbf{E6} & \textbf{E7} & \textbf{E8} & \textbf{E9} & \textbf{E10} \\
\midrule
Gumbel & Self-attention & 1.78e-5 & 9.00e-6 & 5.96e-6 & 4.86e-6 & 3.96e-6 & 3.58e-6 & 3.10e-6 & 2.92e-6 & 2.83e-6 & 2.77e-6 \\
Gaussian & Self-attention & 1.67e-5 & 9.54e-6 & 6.79e-6 & 4.74e-6 & 3.87e-6 & 3.40e-6 & 3.07e-6 & 2.89e-6 & 2.79e-6 & 2.74e-6 \\
Logistic & Self-attention & 1.60e-5 & 8.88e-6 & 6.17e-6 & 4.92e-6 & 4.05e-6 & 3.38e-6 & 3.10e-6 & 2.92e-6 & 2.80e-6 & 2.74e-6 \\
Uniform & Self-attention & 1.96e-5 & 8.34e-6 & 6.08e-6 & 4.74e-6 & 3.84e-6 & 3.38e-6 & 3.04e-6 & 2.88e-6 & 2.77e-6 & 2.74e-6 \\
\midrule
Gumbel & Full block & 3.36e-5 & 1.94e-5 & 1.53e-5 & 1.35e-5 & 1.25e-5 & 1.18e-5 & 1.13e-5 & 1.10e-5 & 1.08e-5 & 1.08e-5 \\
Gaussian & Full block & 3.39e-5 & 1.93e-5 & 1.53e-5 & 1.34e-5 & 1.23e-5 & 1.16e-5 & 1.12e-5 & 1.10e-5 & 1.08e-5 & 1.07e-5 \\
Logistic & Full block & 3.36e-5 & 1.94e-5 & 1.53e-5 & 1.35e-5 & 1.25e-5 & 1.17e-5 & 1.13e-5 & 1.10e-5 & 1.08e-5 & 1.07e-5 \\
Uniform & Full block & 3.39e-5 & 1.94e-5 & 1.53e-5 & 1.34e-5 & 1.23e-5 & 1.16e-5 & 1.11e-5 & 1.09e-5 & 1.08e-5 & 1.07e-5 \\
\bottomrule
\end{tabular}}
\end{table*}

At epoch 10, self-attention losses are between $2.74\times 10^{-6}$ and $2.77\times 10^{-6}$, while full-block losses are between $1.07\times 10^{-5}$ and $1.08\times 10^{-5}$. In the three replicate runs per family, the reported run-to-run half-range was at most $5\times 10^{-8}$. Consequently, these experiments do not reveal a meaningful advantage for a particular variance-matched noise family. They support the use of stochastic exploration, but not a claim that the Gumbel distribution minimizes the reconstruction loss among all perturbations.

\paragraph{Sensitivity to noise magnitude.}
We also multiply each perturbation by $c\in\{0.5,1,2\}$, keeping the other settings fixed. Table~\ref{tab:gsq_noise_scale} gives the final hard losses for each setting. These are single-run results per configuration, and should not be interpreted as a statistically powered comparison of scale values.

\begin{table}[t]
\centering
\small
\setlength{\tabcolsep}{5pt}
\caption{Noise-magnitude ablation. Each entry reports final self-attention / full-block hard validation loss. All noise families have variance $\pi^2/6$ before multiplication by $c$.}
\label{tab:gsq_noise_scale}
\begin{tabular}{lccc}
\toprule
\textbf{Noise family} & $c=0.5$ & $c=1$ & $c=2$ \\
\midrule
Gumbel   & 2.29e-6 / 1.03e-5 & 2.77e-6 / 1.08e-5 & 3.96e-6 / 1.31e-5 \\
Gaussian & 2.25e-6 / 1.02e-5 & 2.74e-6 / 1.07e-5 & 3.90e-6 / 1.29e-5 \\
Logistic & 2.26e-6 / 1.02e-5 & 2.74e-6 / 1.07e-5 & 3.87e-6 / 1.29e-5 \\
Uniform  & 2.25e-6 / 1.01e-5 & 2.74e-6 / 1.07e-5 & 3.96e-6 / 1.31e-5 \\
\bottomrule
\end{tabular}
\end{table}

The magnitude has a substantially larger effect than the family: reducing $c$ from $2$ to $0.5$ improves the final loss consistently for all four distributions. Indeed, for probabilities of the form $p_i\propto\exp((\kappa\ell_i+cg_i)/\tau)$, replacing $(\kappa,\tau,c)$ with $(\kappa/c,\tau/c,1)$ leaves the probabilities unchanged when $c>0$. Hence, the noise-magnitude study also probes the effective logit and temperature schedule. The consistent advantage of smaller perturbations in this controlled setting suggests possible schedule improvements, rather than an advantage of switching noise distributions. We retain Gumbel noise for its exact categorical-sampling interpretation.

\subsection{Hyperparameter sensitivity and optimization trajectory}
\label{app:gsq_hparam_sensitivity}

We conduct a one-at-a-time sensitivity study on the first transformer block of Llama-3.1-8B-Instruct at 2 bits, varying the temperature schedule, logit scale schedule, and initialization strength while keeping all remaining settings fixed. The default temperature decreases linearly from $2$ to $0.05$, the logit scale increases linearly from $100$ to $500$, and the initialization strength is $\alpha=6$, consistent with Table~\ref{tab:hparams_llama}. Table~\ref{tab:gsq_hparam_sensitivity} shows the final hard validation losses; the default optimization trajectory across epochs appears in the Gumbel--Softmax rows of Table~\ref{tab:gsq_estimator_ablation}.

\begin{table}[!t]
\centering
\small
\setlength{\tabcolsep}{7pt}
\caption{One-at-a-time sensitivity study for 2-bit \methodname{} on the first Llama-3.1-8B-Instruct block. Entries are final hard validation losses (lower is better). The default schedule uses $\tau:2\to0.05$, $\kappa:100\to500$, and $\alpha=6$.}
\label{tab:gsq_hparam_sensitivity}
\begin{tabular}{llcc}
\toprule
\textbf{Varied parameter} & \textbf{Setting} & \textbf{Self-attention} & \textbf{Full block} \\
\midrule
--- & Default & 1.94e-6 & \textbf{9.27e-6} \\
Temperature & $4\to0.05$ & \textbf{1.92e-6} & 9.46e-6 \\
Temperature & $1\to0.05$ & 2.31e-6 & 1.03e-5 \\
Logit scale & $10\to50$ & 6.00e-6 & 2.24e-5 \\
Logit scale & $1000\to5000$ & 2.96e-6 & 1.45e-5 \\
Initialization strength $\alpha$ & $3$ & 2.23e-6 & 9.58e-6 \\
Initialization strength $\alpha$ & $9$ & \textbf{1.92e-6} & 9.39e-6 \\
\bottomrule
\end{tabular}
\end{table}

The method is reasonably stable to the tested changes in temperature and initialization strength, but more sensitive to the logit scale schedule. Increasing $\alpha$ from $6$ to $9$ or increasing the initial temperature from $2$ to $4$ slightly improves self-attention reconstruction, while the default gives the best full-block result. By contrast, reducing the logit scale schedule to $10\to50$ raises the full-block loss to $2.24\times10^{-5}$. Overall, the default schedule offers a balanced choice and is used unchanged in the main experiments. These experiments concern a single block and do not establish optimality across all model sizes or training budgets.

\end{document}